\documentclass[12pt]{article}

\usepackage{amsmath}
\usepackage{times}
\usepackage[dvipsone]{graphicx}
\usepackage{color}
\usepackage{multirow}
\usepackage[authoryear]{natbib}
\usepackage{rotating}
\usepackage{bbm}
\usepackage{latexsym}
\usepackage{epstopdf}   

\makeatletter
\@addtoreset{equation}{section}
\makeatother
\renewcommand{\theequation}{\arabic{section}.\arabic{equation}}

\newcommand{\captionfonts}{\normalsize}

\makeatletter
\long\def\@makecaption#1#2{%
  \vskip\abovecaptionskip
  \sbox\@tempboxa{{\captionfonts #1: #2}}%
  \ifdim \wd\@tempboxa >\hsize
    {\captionfonts #1: #2\par}
  \else
    \hbox to\hsize{\hfil\box\@tempboxa\hfil}%
  \fi
  \vskip\belowcaptionskip}
\makeatother

\usepackage{amsmath}
\usepackage{mathrsfs}

\usepackage{algorithm}
\usepackage{algorithmic}
\makeatletter
\renewcommand{\fnum@algorithm}{\fname@algorithm}
\makeatother

\usepackage{amsthm}

\usepackage{pmat}
\usepackage{subfig}
\usepackage{caption}

\usepackage{amssymb}

\usepackage[retainorgcmds]{IEEEtrantools}

\renewcommand{\theequationdis}{{\normalfont (\theequation)}}
\renewcommand{\theIEEEsubequationdis}{{\normalfont (\theIEEEsubequation)}}

\makeatletter
\newcommand*{\rom}[1]{\expandafter\@slowromancap\romannumeral #1@}
\makeatother

\newtheorem{thm}{Theorem}
\newtheorem{lem}{Lemma}
\newtheorem{dfn}{Definition}
\newtheorem{prp}{Proposition}
\newtheorem*{rmk}{Remark}
\newtheorem{rmk-2}{Remark}
\newtheorem{rmk-3}{Remark}
\newtheorem{rmk-4}{Remark}
\newtheorem{rmk-5}{Remark}
\newtheorem{rmk-6}{Remark}
\newtheorem{rmk-7}{Remark}
\newtheorem{rmk-8}{Remark}
\newtheorem{cl}{Corollary}

\newtheorem{assm}{Assumption}

\usepackage[hyperindex,breaklinks]{hyperref}
\hypersetup{
           breaklinks=true,   
           colorlinks=true,   
           pdfusetitle=true,  
        }
\usepackage{setspace}
\begin{document}

\ \vspace{20mm}\\

{\LARGE \flushleft Universal Solutions of Feedforward ReLU Networks for Interpolations}

\ \\
{\bf \large Changcun Huang}\\
{cchuang@mail.ustc.edu.cn}\\
%


\thispagestyle{empty}
\markboth{}{NC instructions}
\ \vspace{-0mm}\\
%
\begin{center} {\bf Abstract} \end{center}
This paper provides a theoretical framework on the solution of feedforward ReLU networks for interpolations, in terms of what is called an interpolation matrix, which is the summary, extension and generalization of our three preceding works, with the expectation that the solution of engineering could be included in this framework and finally understood. To three-layer networks, we classify different kinds of solutions and model them in a normalized form; the solution finding is investigated by three dimensions, including data, networks and the training; the mechanism of a type of overparameterization solution is interpreted. To deep-layer networks, we present a general result called sparse-matrix principle, which could describe some basic behavior of deep layers and explain the phenomenon of the sparse-activation mode that appears in engineering applications associated with brain science; an advantage of deep layers compared to shallower ones is manifested in this principle. As applications, a general solution of deep neural networks for classifications is constructed by that principle; and we also use the principle to study the data-disentangling property of encoders. Analogous to the three-layer case, the solution of deep layers is also explored through several dimensions. The mechanism of multi-output neural networks is explained from the perspective of interpolation matrices.


\ \\[-2mm]

\section{Introduction}
A ReLU \citep*{Nair2010,Glorot2011} has become the most popular unit of neural networks in recent years \citep*{Goodfellow2016,LeCun2015}. To the solutions of the classification or approximation via ReLU networks, besides those derived from the training method such as in convolutional neural networks \citep*{Krizhevsky2017} and autoencoders \citep*{Hinton2006,Girin2020}, there also exist some constructed ones, which can be classified by their original ideas as: polynomial approximations \citep*{Yarotsky2017,Liang2017,Telgarsky2015}, wavelet constructions \citep*{Daubechies2019,Shaham2018, Huang2020}, hinging-hyperplane methods \citep*{Arora2018}, and piecewise linear approximations \citep*{Shen2021,Huang2020}.

The two directions above are nearly parallel to each other. Instead of the training method, \citet*{Huang2022a} attempted to find a solution of engineering through theoretical constructions. Although several ways were proposed to enhance the universality of the constructed solution, there are still some obstacles rendering it too specially designed, such that it may not be easily encountered by the training process. Especially, the interference-avoiding principle needs a certain type of hyperplane arrangement, which is a constraint.

This paper tries to explore more universal solutions with the help of algebraic methods, on the basis of \citet*{Huang2022a,Huang2022b, Huang2020}. The most related works to ours are \citet*{zhang2017} and \citet*{Collins2018}. The former had used the interpolation matrix for ReLU networks, but only involved a special case relevant to their concrete interpolation strategies. The latter studied the nonlinear effect of zero entries of an interpolation matrix, through an extracted feature that can be quantitatively measured; however, this measure can only grasp some rough property of interpolation matrices.

We will systematically investigate the property of interpolation matrices in this paper, by which a framework of universal solutions of feedforward ReLU networks will be established. It is hoped that the solution obtained by the training method could be explained under this framework.

\subsection{Arrangements and Contributions}
Section 2 introduces an interpolation framework and presents some basic principles (theorem 1 and lemma 2) on the effect of zero entries of an interpolation matrix. Section 3 classifies various kinds of solutions (definitions 2, 3 and 4) and introduces a normalized form to describe them (theorem 2). Proposition 1 lists some factors that could lead to the solution mode of theorem 2.

Section 4 analyses the source of multiple solutions (theorem 3) and studies them in terms of interpolation matrices (theorem 4), with the mechanism of overparameterization solutions interpreted simultaneously.

Section 5 integrates the training procedure into the solution finding (theorem 7). Lemmas 6, 7 and 8 are general principles for the training of neural networks. The influence of the training on interpolation matrices is given in theorem 5.

Section 6 constructs a solution of deep-layer networks for binary classification and data disentangling, whose mechanism (theorem 8) is novel and different from that of \citet*{Huang2022a,Huang2022b}. Theorem 8 also severs as both an evidence of the later presented theories and a typical example that is helpful to understand them.

Section 7 gives one of the main results of this paper, which is called \textsl{sparse-matrix principle} (theorem 10), on the basis of a general conclusion (theorem 9). The sparse-activation mode of neural networks, which is related to both engineering applications and brain science, will be interpreted by that principle (theorem 12). We model a deep-layer network by decomposing it into two subnetworks, such that the results of three-layer networks could be applied (lemma 13 and theorem 11).

The concept of the activation route of data (definitions 8 and 9) is proposed, which plays a fundamental role in our analyses. We'll also incorporate the training procedure into the solution finding of deep-layer networks (theorem 11). In section 7.7, the mechanism of ReLU networks for multi-outputs is explained from the perspective of interpolation matrices.

Section 8 is an application of the theories of section 7 in autoencoders. The data-disentangling property of an encoder, which had been intensively studied in \citet*{Huang2022b}, will be interpreted and generalized (theorem 13) by the sparse-matrix principle of theorem 10. Section 9 concludes this paper by a discussion.

\subsection{Notes and Notations}
We will borrow some notations, notes and instructions from \citet*{Huang2022a,Huang2022b}, which are listed below. The notes of \citet*{Huang2022a} are in its section 2.4.

\begin{itemize}
\item[1.] The reason that the interpolation framework is used to model neural networks is explained in section 1.2 of \citet*{Huang2022a}.
\item[2.] Definitions 1 and 2 of \citet*{Huang2022a}: the notation of network architectures, such as $n^{(1)}m^{(1)}{\mu}^{(1)}$.
\item[3.] Note 2 of \citet*{Huang2022a}: the notations relevant to a hyperplane derived from a unit, such as $l^0$, $l^+$ and $l_1^+l_2^+$.
\item[4.] Note 3 of \citet*{Huang2022a}: the index of a layer of neural networks, with that of the input layer being $0$ and the hidden layers starting from $1$.
\item[5.] Definitions 4 and 5 of \citet*{Huang2022a}: the concepts of the activation of units, including ``simultaneously activate'' and ``partially activate''.
\item[6.] Definition 11 of \citet*{Huang2022b}: divided regions of a set of hyperplanes.
\item[7.] Note 7 of \citet*{Huang2022a}: the definition of the output of a hyperplane.
\item[8.] Definition 10 of \citet*{Huang2022a}: the meaning of ``having the same classification effect'' for hyperplanes.
\item[9.] Definition 16 of \citet*{Huang2022a}: the concepts of open and closed convex polytopes.
\item[10.] When we mention the probability measure, the underlying probabilistic model is the one introduced in the appendix of \citet*{Huang2022a}.
\item[11.] Except for section 6, any unit of the output player of a neural network is assumed to be a linear one.
\item[12.] Definition 9 of \citet*{Huang2022b}: the concept of data disentangling.

\end{itemize}

\section{Three-layer ReLU Networks}
The results of sections 2 and 3 are the preliminaries to the whole paper. The mechanism of a three-layer network is the basis of a deep-layer one, since the former is a subnetwork of the latter; this's why we discuss the three-layer case first.

\subsection{Model Description}
We'll use two interpolation frameworks. The first one to be introduced in this section deals with the discrete points coming from a piecewise linear function, whose each linear component provides a finite number of points. We propose this model because it is a typical form of approximations or interpolations in practice, when different batchs of points are fitted or interpolated by their corresponding hyperplanes.

The second framework will be given in section 4.1, which is the usual form that takes all the points to be interpolated as a whole, regardless of whether part of them forming a hyperplane or not, analogous to that of \citet*{Haykin2009} and \citet*{DeVore2021}.

Denote a discrete piecewise linear function by
\begin{equation}
f: D \to \mathbb{R}
\end{equation}
whose domain
\begin{equation}
D = \bigcup_{i=1}^{N}D_i \subset \mathbb{R}^n,
\end{equation}
in which each subdomain $D_i$ is a data set composed of $\beta_i$ discrete points, and a certain linear function or an $n$-dimensional hyperplane is defined on it by the function of equation 2.1. Let
\begin{equation}
\mathcal{C} : =|D| = \sum_{i = 1}^N\beta_i
\end{equation}
be the cardinality of $D$.

Suppose that we have $N$ $n$-dimensional hyperplanes $h_i$'s for $i = 1, 2, \cdots, N$ of $n + 1$-dimensional space, each of which provides $\beta_i$ discrete points $\{\boldsymbol{x}_{ij}, y_{ij}\}$ for $j = 1, 2, \cdots, \beta_i$; all of the data points are collectively denoted by a set
\begin{equation}
F = \bigcup_{i=1}^N F_i,
\end{equation}
where $F_i$ is the set of the points of hyperplane $h_i$. Then equation 2.4 is another expression of a discrete piecewise linear function of equation 2.1, with subdomain
\begin{equation}
D_i = \{\boldsymbol{x}_{i1}, \boldsymbol{x}_{i2}, \cdots, \boldsymbol{x}_{i\beta_i} \}
\end{equation}
associated with $F_i$.

The interpolation of each point of $F_i$ by a three-layer ReLU network $n^{(1)}m^{(1)}1^{(1)}$ can be expressed as
\begin{equation}
\sum_{\nu = 1}^{m}\alpha_{\nu}\Phi_{\nu}(\boldsymbol{x}_{ij}) = y_{ij},
\end{equation}
where $\alpha_{\nu}$ is the output weight of the unit $u_{\nu}$ of the hidden layer, or the $\nu$th input weight of the unit of the output layer, and
\begin{equation}
\Phi_{\nu}(\boldsymbol{x}_{ij}) := \sigma(\boldsymbol{w}_{\nu}^T\boldsymbol{x}_{ij} + b_{\nu}),
\end{equation}
where $\sigma(s) = \max(0,s)$ is the activation function of a ReLU, while $\boldsymbol{w}_{\nu}$ and $b_{\nu}$ are the parameters of $u_{\nu}$. The case of all $F_i$'s could be integrated into a matrix form as
\begin{equation}
\boldsymbol{\Phi}\boldsymbol{\alpha} = \boldsymbol{y},
\end{equation}
where
\begin{equation}
\boldsymbol{\Phi}=\begin{bmatrix}
\Phi_1(\boldsymbol{x}_{11}) & \Phi_2(\boldsymbol{x}_{11}) & \cdots & \Phi_m(\boldsymbol{x}_{11})\\
\vdots & \vdots &\ddots & \vdots\\
\Phi_1(\boldsymbol{x}_{1\beta_1}) & \Phi_2(\boldsymbol{x}_{1\beta_1}) & \cdots & \Phi_m(\boldsymbol{x}_{1\beta_1})\\
\vdots & \vdots &\ddots & \vdots\\
\Phi_1(\boldsymbol{x}_{N1}) & \Phi_2(\boldsymbol{x}_{N1}) & \cdots & \Phi_m(\boldsymbol{x}_{N1})\\
\vdots & \vdots &\ddots & \vdots\\
\Phi_1(\boldsymbol{x}_{N\beta_N}) & \Phi_2(\boldsymbol{x}_{N\beta_N}) & \cdots & \Phi_m(\boldsymbol{x}_{N\beta_N})\\
\end{bmatrix}
\end{equation}
with $\boldsymbol{\alpha} = [\alpha_1, \alpha_2, \cdots, \alpha_m]^T$, and $\boldsymbol{y} = [y_{11}, \cdots, y_{1\beta_1}, \cdots, y_{N1}, \cdots, y_{N\beta_N}]^T$. The size of $\boldsymbol{\Phi}$ is $\mathcal{C} \times m$, where $\mathcal{C}$ defined in equation 2.3 is the number of the data points and $m$ is the number of the units of the hidden layer.

Write equation 2.9 as
\begin{equation}
\boldsymbol{\Phi}=\begin{bmatrix}
\boldsymbol{H}_1^T, & \boldsymbol{H}_2^T, & \cdots, & \boldsymbol{H}_N^T
\end{bmatrix}^T,
\end{equation}
where
\begin{equation}
\boldsymbol{H}_i=\begin{bmatrix}
\Phi_1(\boldsymbol{x}_{i1}) & \Phi_2(\boldsymbol{x}_{i1}) & \cdots & \Phi_m(\boldsymbol{x}_{i1})\\
\Phi_1(\boldsymbol{x}_{i2}) & \Phi_2(\boldsymbol{x}_{i2}) & \cdots & \Phi_m(\boldsymbol{x}_{i2})\\
\vdots & \vdots &\ddots & \vdots\\
\Phi_1(\boldsymbol{x}_{i\beta_i}) & \Phi_2(\boldsymbol{x}_{i\beta_i}) & \cdots & \Phi_m(\boldsymbol{x}_{i\beta_i})\\
\end{bmatrix}
\end{equation}
represents the interpolation of $F_i$ of equation 2.4 for $i = 1, 2, \cdots, N$, and the size $\beta_i \times m$ means that use $m$ units of the hidden layer to interpolate the $\beta_i$ points of $F_i$. Under equation 2.10, equation 2.8 can be decomposed into
\begin{equation}
\boldsymbol{H}_i\boldsymbol{\alpha} = \boldsymbol{y}_i
\end{equation}
with $\boldsymbol{y}_i = [y_{i1}, y_{i2}, \cdots, y_{i\beta_i}]^T$ for all $i$.

Another form of equation 2.9 is
\begin{equation}
\boldsymbol{\Phi}=\begin{bmatrix}
\boldsymbol{L}_1, & \boldsymbol{L}_2, & \cdots, & \boldsymbol{L}_m
\end{bmatrix},
\end{equation}
where
\begin{equation}
\boldsymbol{L}_{\nu} = \begin{bmatrix}
\Phi_{\nu}(\boldsymbol{x}_{11}), \cdots, \Phi_{\nu}(\boldsymbol{x}_{1\beta_1}), \cdots, \Phi_{\nu}(\boldsymbol{x}_{N1}), \cdots, \Phi_{\nu}(\boldsymbol{x}_{N\beta_N})
\end{bmatrix}^T
\end{equation}
for $\nu = 1, 2, \cdots, m$, which is the $\nu$th column of equation 2.9, whose entries are the outputs of $u_{\nu}$ of the hidden layer with respect to the input data set $D$.

\begin{dfn}
We call $\boldsymbol{\Phi}$ of equation 2.9 the interpolation matrix of a three-layer network $n^{(1)}m^{(1)}1^{(1)}$ with respect to an input data set $D$.
\end{dfn}

\subsection{Basic Principles}

\begin{lem}
If $D$ of equation 2.1 simultaneously activates all the $m$ units of the hidden layer of $n^{(1)}m^{(1)}1^{(1)}$, and if the interpolation matrix $\boldsymbol{\Phi}$ of equation 2.9 satisfies $m = \mathcal{C} > n+1$, then $\boldsymbol{\Phi}$ is singular.
\end{lem}
\begin{proof}
Let ${\lambda}_{\nu} = [\boldsymbol{w}_{\nu}^T, b_{\nu}]^T$ and $\boldsymbol{z}_{ij} = [\boldsymbol{x}_{ij}^T, 1]^T$. Then $\boldsymbol{w}_{\nu}^T\boldsymbol{x}_{ij} + b_{\nu}$ of equation 2.7 can be expressed as
\begin{equation}
\boldsymbol{w}_{\nu}^T\boldsymbol{x}_{ij} + b_{\nu} = \boldsymbol{\lambda}_{\nu}^T\boldsymbol{z}_{ij},
\end{equation}
where $\boldsymbol{\lambda}_{\nu}$ is an $n+1$-dimensional column vector. Due to the condition of this lemma, each element of $\boldsymbol{\Phi}$ is nonzero, and thus equation 2.14 is equivalent to
\begin{equation}
\boldsymbol{L}_{\nu} = \begin{bmatrix}
\boldsymbol{\lambda}_{\nu}^T\boldsymbol{z}_{11}, \cdots, \boldsymbol{\lambda}_{\nu}^T\boldsymbol{z}_{1\beta_1}, \cdots, \boldsymbol{\lambda}_{\nu}^T\boldsymbol{z}_{N1}, \cdots, \boldsymbol{\lambda}_{\nu}^T\boldsymbol{z}_{N\beta_N}
\end{bmatrix}^T.
\end{equation}

Because $\boldsymbol{\lambda}_{\nu}$ is $n+1$-dimensional, among $\boldsymbol{\lambda}_{\nu}$'s for $\nu = 1, 2, \cdots, m$ with $m > n+1$, there are at most $n+1$ linearly independent vectors. Without loss of generality, we assume that $\boldsymbol{\lambda}_1, \boldsymbol{\lambda}_2, \cdots, \boldsymbol{\lambda}_{n+1}$ are linearly independent; then each of $\boldsymbol{\lambda}_{n+2}, \boldsymbol{\lambda}_{n+3}, \cdots, \boldsymbol{\lambda}_{m}$ is the linear combination of the first $n+1$ ones. For example, to a certain $\boldsymbol{\lambda}_{t}$ for $t \ge n+2$, we have
\begin{equation}
\boldsymbol{\lambda}_t = \boldsymbol{a}^T\boldsymbol{e},
\end{equation}
where $\boldsymbol{a} = [a_1, a_2, \cdots, a_{n+1}]^T$ and $\boldsymbol{e} = [\boldsymbol{\lambda}_1^T, \boldsymbol{\lambda}_2^T, \cdots, \boldsymbol{\lambda}_{n+1}^T]^T$.

Substituting equation 2.17 into equation 2.16, we have
\begin{equation}
\begin{aligned}
\boldsymbol{L}_t &= \begin{bmatrix}
(\boldsymbol{a}^T\boldsymbol{e})\boldsymbol{z}_{11}, \cdots, (\boldsymbol{a}^T\boldsymbol{e})\boldsymbol{z}_{1\beta_1}, \cdots, (\boldsymbol{a}^T\boldsymbol{e})\boldsymbol{z}_{N1}, \cdots, (\boldsymbol{a}^T\boldsymbol{e})\boldsymbol{z}_{N\beta_N}
\end{bmatrix}^T \\
&= \sum_{\mu = 1}^{n+1}a_{\mu}
\begin{bmatrix}
\boldsymbol{\lambda}_{\mu}^T\boldsymbol{z}_{11}, \cdots, \boldsymbol{\lambda}_{\mu}^T\boldsymbol{z}_{1\beta_1}, \cdots, \boldsymbol{\lambda}_{\mu}^T\boldsymbol{z}_{N1}, \cdots, \boldsymbol{\lambda}_{\mu}^T\boldsymbol{z}_{N\beta_N}
\end{bmatrix}^T = \sum_{\mu = 1}^{n+1}a_{\mu}\boldsymbol{L}_{\mu},
\end{aligned}
\end{equation}
where $a_{\mu}$'s are the entries of $\boldsymbol{a}$ of equation 2.17. Equation 2.18 means that the $t$th column of $\boldsymbol{\Phi}$ for $t \ge n+2$ is the linear combination of the first $n+1$ columns, which implies the singularity of $\boldsymbol{\Phi}$.

If we cannot find $n+1$ linearly independent vectors among $\boldsymbol{\lambda}_{\nu}$'s, similar proofs could be done in a lower-dimensional space. This completes the proof.
\end{proof}

\begin{lem}
To the interpolation matrix $\boldsymbol{\Phi}$ of equation 2.9, suppose that $m = \mathcal{C}$. The singular or nonsingular property of $\boldsymbol{\Phi}$ could be changed by introducing zero entries, through the control of the activations of the hidden-layer units.
\end{lem}
\begin{proof}
In equation 2.18 of the proof of lemma 1, if the $t$th column $\boldsymbol{L}_t$ of $\boldsymbol{\Phi}$ has zero entries due to the zero-output property of a ReLU instead of the linear combination of $\boldsymbol{L}_{\mu}$'s, the result $\boldsymbol{L}_t = \sum_{\mu = 1}^{n+1}a_{\mu}\boldsymbol{L}_{\mu}$ would not hold, and the singularity of $\boldsymbol{\Phi}$ due to lemma 1 may be changed.

To the existence of a nonsingular $\boldsymbol{\Phi}$ by introducing zero entries, we generalize the solution of lemma 4 of \citet*{Huang2022a}, in which each subdomain $D_i$ of equation 2.2 contains only one element. After $D_i$'s having been distinguished via lemma 4 of \citet*{Huang2022a}, to each $\boldsymbol{x}_i$ of $D_i$, in its sufficiently small neighbourhood, we select other $n$ points that are in the same divided region as $\boldsymbol{x}_i$, which comprise the new subdomain $D_i$ along with $\boldsymbol{x}_i$. Then the enlarged $D_i$'s are also distinguishable, with $|D_i| = n+1$ and $\mathcal{C} = N(n+1) > n+1$.

We assume that to each $D_i$, the $n+1$ points of the corresponding $F_i$ of equation 2.4 determine a unique $n$-dimensional hyperplane $h_i$. This condition is easily satisfied because the construction of those points is trivial.

Then according to each distinguishable hyperplane $l_i$ of $D_i$, use proposition 1 of \citet*{Huang2022a} to construct other $n$ hyperplanes that have the same classification effect as $l_i$, all of which together with $l_i$ are collectively denoted by $L_i$. Then $m = \mathcal{C} = N(n+1)$. Finally, the corresponding interpolation matrix is
\begin{equation}
\boldsymbol{\Phi}_d=\begin{bmatrix}
\boldsymbol{P}_{11} & \boldsymbol{0} & \boldsymbol{0}  & \cdots & \boldsymbol{0}\\
\boldsymbol{U}_{21} & \boldsymbol{P}_{22} & \boldsymbol{0}  & \cdots & \boldsymbol{0}\\
\vdots & \vdots & \vdots &\ddots & \vdots\\
\boldsymbol{U}_{N1} & \boldsymbol{U}_{N2} & \boldsymbol{U}_{N3} & \cdots & \boldsymbol{P}_{NN}\\
\end{bmatrix},
\end{equation}
where submatrix $\boldsymbol{P}_{ii}$ is a positive-entry matrix of size $(n+1) \times (n+1)$ that corresponds to the set $L_i$ as mentioned above, and $\boldsymbol{U}_{ij}$ is a matrix whose entries are either zero or nonzero. Then
\begin{equation}
\det\boldsymbol{\Phi} = \prod_{i=1}^N\det\boldsymbol{P}_{ii}.
\end{equation}
Under equation 2.19, equation 2.8 can be written as
\begin{equation}
\boldsymbol{\Phi}_d\begin{bmatrix}
\boldsymbol{\alpha}_{1}^T, \boldsymbol{\alpha}_{2}^T, \cdots, \boldsymbol{\alpha}_{N}^T
\end{bmatrix}^T = \boldsymbol{y},
\end{equation}
where $\boldsymbol{\alpha}_{i}$ for $i = 1, 2, \cdots, N$ is an $(n+1) \times 1$ vector whose entries are the $n+1$ output weights of the units of the hidden layer associated with $L_i$.

Now prove $\det\boldsymbol{P}_{ii} \ne 0$. Under this lemma, decompose $\boldsymbol{y}$ of equation 2.8 into the form $\boldsymbol{y} = [\boldsymbol{y}_1^T, \boldsymbol{y}_2^T, \cdots, \boldsymbol{y}_N^T]^T$, where $\boldsymbol{y}_i$ of size $(n+1) \times 1$ corresponds to the function values of subdomain $D_i$ of equation 2.2. By equations 2.19 and 2.21, we have
\begin{equation}
\begin{aligned}
& \boldsymbol{P}_{11}\boldsymbol{\alpha}_1 = \boldsymbol{y}_1, \\
\boldsymbol{P}_{jj}\boldsymbol{\alpha}_j = & \ \boldsymbol{y}_j - \sum_{k = 1}^{j-1}\boldsymbol{U}_{j, j-k}\boldsymbol{\alpha}_{j-k}
\end{aligned}
\end{equation}
for $j = 2, 3, \cdots, N$. Equation 2.22 would give a unique solution of $\boldsymbol{\alpha}_{i}$ for $i = 1, 2, \cdots, N$ for the interpolation of $F_i$ of equation 2.4, if $\det\boldsymbol{P}_{ii} \ne 0$, and vice versa. Since we have assumed that the unique $h_i$ exists for $F_i$, the solution of $\boldsymbol{\alpha}_{i}$ is unique, which implies $\det\boldsymbol{P}_{ii} \ne 0$. Thus, $\det\boldsymbol{\Phi} \ne 0$.

\end{proof}

\begin{thm}
Under the condition of lemma 2, to make the interpolation matrix $\boldsymbol{\Phi}$ of equation 2.9 nonsingular, a necessary condition is that we cannot find $k$ columns for $k \ge n+2$ from equation 2.13 whose entries are all nonzero.
\end{thm}
\begin{proof}
The conclusion is a direct consequence of lemma 1. Any $k$ columns of $\boldsymbol{\Phi}$ for $k \ge n+2$ that have no zero elements would lead to the singularity of $\boldsymbol{\Phi}$.
\end{proof}

\section{Activation Mode of Solutions}
This section will classify the solutions of a three-layer network and model them in a uniform way, which is the basis of the solution finding in later sections.

\subsection{Preliminaries}
The next two definitions are the typical interpolation forms of $F = \bigcup_{i=1}^N F_i$ of equation 2.4 for a single subdomain.
\begin{dfn}
Under equation 2.4, suppose that cardinality $|F_i| \ge n+1$ for $i = 1, 2, \cdots, N$ , and that the points of $F_i$ determine a unique $n$-dimensional hyperplane $h_i$ of $n+1$-dimensional space. If network $n^{(1)}m^{(1)}1^{(1)}$ interpolates $F_i$ by $h_i$, we say that it is an exact interpolation.
\end{dfn}

\begin{dfn}
Under the notations of definition 2, suppose that the points of $F_i$ are on a lower-dimensional hyperplane of $n+1$-dimensional space, whose dimensionality is less than $n$. To network $n^{(1)}m^{(1)}1^{(1)}$, the interpolation of the lower-dimensional $F_i$ is in terms of a higher $n$-dimensional hyperplane, for which we call it a redundant interpolation.
\end{dfn}

\begin{figure}[!t]
\captionsetup{justification=centering}
\centering
\includegraphics[width=2.3in, trim = {4.8cm 3.3cm 3.5cm 3.4cm}, clip]{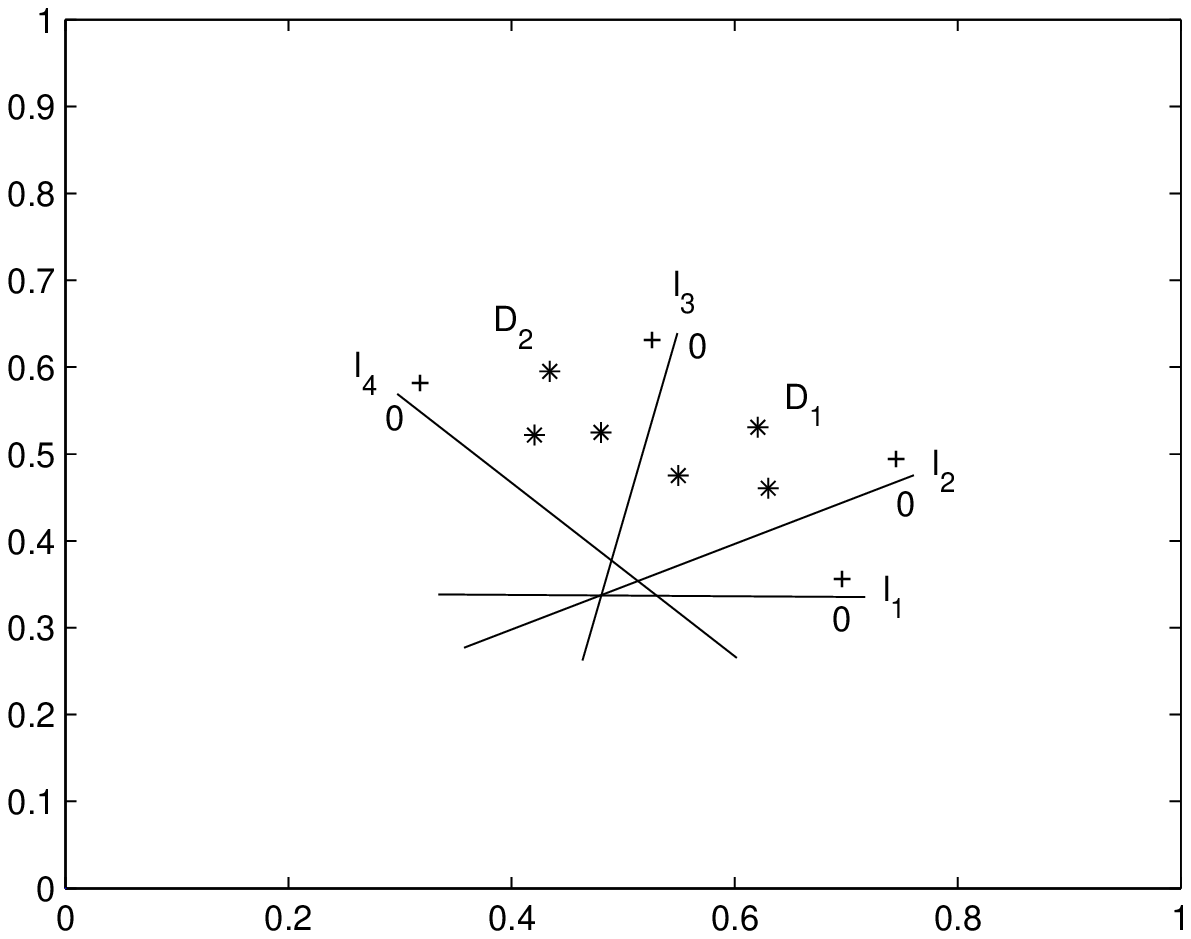}
\caption{Dependent piecewise linear interpolation.}
\label{Fig.1}
\end{figure}

In equation 2.4, the interpolation of $F_i$ may be independent of other subsets of $F$, or be correlated with some of them. The former case had been studied in lemmas 1, 2, and theorem 2 of \citet*{Huang2022a}, while the latter case has not been mentioned yet. We want to present a general model that can include both the above two cases. In order for that, first give their formal definitions.

\begin{lem}
To interpolate $F = \bigcup_{i=1}^N F_i$ of equation 2.4 by network $n^{(1)}m^{(1)}1^{(1)}$, some subsets among $F_i$'s may be correlated with each other, such that one of them cannot be interpolated by an arbitrary linear function, due to the parameter sharing between the associated subdomains $D_i$'s of equation 2.5.
\end{lem}
\begin{proof}
This conclusion is obvious by an example of Figure \ref{Fig.1}. To the implementation of an arbitrary linear function on $D_1$ or $D_2$, by lemma 1 of \citet*{Huang2022a}, at least three activated lines (units) are needed. However, as shown in Figure \ref{Fig.1}, $D_1$ shares three activated lines $l_1$, $l_2$ and $l_4$ with $D_2$, and the difference is only in line $l_3$. If the three activated lines that they share are used to produce a certain linear function on $D_1$, $D_2$ could only resort to $l_3$, and hence the linear function on $D_2$ is restricted. The general case is similar.
\end{proof}

\begin{dfn}
An independent piecewise linear interpolation of $F = \bigcup_{i=1}^NF_i$ of equation 2.4 is the one that the correlated case of lemma 3 does not happen; otherwise, we call it a dependent piecewise linear interpolation.
\end{dfn}

We change the interpolation matrix $\boldsymbol{\Phi}$ of equation 2.9 into the form
\begin{equation}
\boldsymbol{\Phi}_B = \boldsymbol{\Phi} = \begin{bmatrix}
\boldsymbol{B}_{11} & \boldsymbol{B}_{12} & \cdots & \boldsymbol{B}_{1M}\\
\boldsymbol{B}_{21} & \boldsymbol{B}_{22} & \cdots & \boldsymbol{B}_{2M}\\
\vdots & \vdots &\ddots & \vdots\\
\boldsymbol{B}_{N1} & \boldsymbol{B}_{N2} & \cdots & \boldsymbol{B}_{NM}
\end{bmatrix},
\end{equation}
where the column partitions the $m$ hyperplanes (or units) into $M$ groups, the row partitions the $\mathcal{C}$ points analogously, and the block $\boldsymbol{B}_{\nu\mu}$ for $1 \le \nu \le N$ and $1 \le \mu \le M$ is a matrix whose entries are the outputs of the $\nu$th group hyperplane with respect to the $\mu$th subdomain $D_{\mu}$.

\begin{dfn}
The activation mode of an interpolation matrix in the form of equation 3.1 is defined as
\begin{equation}
\mathcal{M} = \big(m_{\nu\mu}\big),
\end{equation}
where $m_{\nu\mu} = \boldsymbol{P}$, $\boldsymbol{0}$, $\boldsymbol{U}$, or $\boldsymbol{0}'$ for $1 \le \nu \le N$ and $1 \le \mu \le M$, The first three cases of the value of $m_{\nu\mu}$ correspond to positive-entry $\boldsymbol{B}_{\nu\mu}$, zero-entry $\boldsymbol{B}_{\nu\mu}$, and uncertain-entry (either positive or zero) $\boldsymbol{B}_{\nu\mu}$ of equation 3.1, respectively. The last case $\boldsymbol{0}'$ means that the corresponding $\boldsymbol{B}_{\nu\mu}$ has at least one column whose entries are all zero. We also call $\mathcal{M}$ the activation-mode matrix of equation 3.1.
\end{dfn}

\noindent
\textbf{Example}. The activation mode of equation 2.19 derived from distinguishable data sets is
\begin{equation}
\mathcal{M}^* = \begin{bmatrix}
\boldsymbol{P} & \boldsymbol{0} & \boldsymbol{0}  & \cdots & \boldsymbol{0}\\
\boldsymbol{U} & \boldsymbol{P} & \boldsymbol{0}  & \cdots & \boldsymbol{0}\\
\vdots & \vdots & \vdots &\ddots & \vdots\\
\boldsymbol{U} & \boldsymbol{U} & \boldsymbol{U} & \cdots & \boldsymbol{P}
\end{bmatrix}.
\end{equation}

\begin{dfn}
If we say that data set $D$ activates a set $H$ of hyperplanes, it means that all the elements of $H$ is activated by each point of $D$.
\end{dfn}

\begin{assm}
Let $l_i$ for $i = 1, 2, \cdots, m$ be the hyperplane corresponding to the $i$th unit of the hidden layer of network $n^{(1)}m^{(1)}1^{(1)}$. Under equation 2.1, suppose that to some subdomain $D_i$, we have $f(D_i) = 0$. Then if the divided region $\prod_{i=1}^{m}l_i^0$ exists, we will not consider the trivial solution of $D_i \subset \prod_{i=1}^{m}l_i^0$.
\end{assm}

\begin{lem}
To the interpolation matrix of equation 2.9 or 3.1, suppose that $\beta_i = |F_i| = n+1$ for $i = 1, 2, \cdots, N$ and $m = \mathcal{C} = N(n+1)$. In terms of the activation mode of equation 3.2, equation 3.3 is the unique solution of an independent piecewise linear interpolation for equation 2.1 or 2.4.
\end{lem}
\begin{proof}
By lemmas 1 and 2 of \citet*{Huang2022a}, there are two principles for the construction of an independent piecewise linear interpolation. First, each subdomain $D_i$ should have at least $n+1$ activated hyperplanes. Second, given arbitrary two subdomains $D_{\nu}$ and $D_{\mu}$ for $1 \le \nu, \mu \le N$ with $\nu \ne \mu$, $D_{\nu}$ should be different from $D_{\mu}$ in at least $n+1$ activated hyperplanes.

Let $H$ be the set of $m$ hyperplanes associated with the units of the hidden layer of $n^{(1)}m^{(1)}1^{(1)}$. The set $H$ could be decomposed into
\begin{equation}
H = \bigcup_{i=1}^NH_i,
\end{equation}
with $|H_i| = n+1$ and $H_{\nu} \cap H_{\mu} = \emptyset$ if $\nu \ne \mu$ for $\nu, \mu = 1,2, \cdots, N$. Correspondingly, $M = N$ in equation 3.1 and $\boldsymbol{\Phi}_B$ is an $N \times N$ block matrix with each block an $(n+1) \times (n+1)$ matrix. The decomposition method of equation 3.4 is as follows.

We examine the activation mode by the above two principles. For instance, to subdomains $D_1$ and $D_2$ with respect to the first two rows of $\boldsymbol{\Phi}_B$ of equation 3.1, in order to satisfy those principles, make $D_1$ and $D_2$ only activate $H_1$ and $H_2$ of equation 3.4, respectively. Then to the activation-mode matrix $\mathcal{M}$ of equation 3.2, its element $m_{11} = \boldsymbol{P}$ and $m_{22} = \boldsymbol{P}$. And there are two ways to distinguish between $D_2$ and $D_1$: $D_2$ doesn't activate $H_1$ implying $m_{21} = \boldsymbol{0}$, or $D_1$ doesn't activate $H_2$ such that $m_{12} = \boldsymbol{0}$.

In general, there are two properties of $\mathcal{M}$ for this lemma: (1) For each $D_i$, the $i$th row of $\mathcal{M}$ has one $\boldsymbol{P}$ element; (2) $D_{\nu}$ is different from $D_{\mu}$ at least either by $m_{\nu\mu} = \boldsymbol{0}$ or by $m_{\mu\nu} = \boldsymbol{0}$ for arbitrary $1 \le \nu, \mu \le N$.

For example, when $N =4$, we construct two activation-mode matrices that satisfy the above properties as:
\begin{equation}
\mathcal{M}'_4 = \begin{bmatrix}
\boldsymbol{P} & \boldsymbol{U} & \boldsymbol{0}  & \boldsymbol{U}\\
\boldsymbol{0} & \boldsymbol{P} & \boldsymbol{0}  & \boldsymbol{0}\\
\boldsymbol{U} & \boldsymbol{U} & \boldsymbol{P}  & \boldsymbol{U}\\
\boldsymbol{0} & \boldsymbol{U} & \boldsymbol{0}  & \boldsymbol{P}
\end{bmatrix},
\mathcal{M}_4 = \begin{bmatrix}
\boldsymbol{P} & \boldsymbol{0} & \boldsymbol{0}  & \boldsymbol{0}\\
\boldsymbol{U} & \boldsymbol{P} & \boldsymbol{0}  & \boldsymbol{0}\\
\boldsymbol{U} & \boldsymbol{U} & \boldsymbol{P}  & \boldsymbol{0}\\
\boldsymbol{U} & \boldsymbol{U} & \boldsymbol{U}  & \boldsymbol{P}
\end{bmatrix}.
\end{equation}

Now we prove that in equation 3.5, $\mathcal{M}_4$ is equivalent to $\mathcal{M}_4'$ from the perspective of a relation underlying the activation-mode matrix. To fixed data $D = \bigcup_{i=1}^{N}D_i$ and hyperplanes $H = \bigcup_{j=1}^{m}H_i$, the relation whether $D_i \subset l_j^+$ or $D_i \subset l_j^0$ with $l_j \in H$ for each $i$ and $j$ is determined, which will be called relation $R$, describing the mutual relationships of the points and hyperplanes of $D$ and $H$, respectively.

Note that $\mathcal{M}_4$ and $\mathcal{M}_4'$ of equation 3.5 are from the same relation $R$, since they both satisfy the above two properties. The distinction between them could be resolved by exchanging the subscripts of $D_i$'s and $L_j$'s, without influencing the underlying relation $R$ because only notations are modified. And the subscript exchanges of $D_i$'s and $L_j$'s correspond to the row swap and column swap of $\mathcal{M}$, respectively, which could facilitate this transform. In fact, for example, we can transform $\mathcal{M}_4'$ to $\mathcal{M}_4$ by row and column swap operations.

The general case is similar, and the standard solution mode could be expressed in the form of equation 3.3. Since $m = N(n+1)$, each subdomain $D_i$ for $i = 1, 2, \cdots, N$ should have exactly $n+1$ activated hyperplanes, and the construction method above is the only way to realize an independent piecewise linear interpolation, which corresponds to the decomposition method of set $H$ of equation 3.4. This completes the proof.

\end{proof}

\subsection{General Results}
\begin{assm}
Using the notations of assumption 1, we assume that subdomains $D_i$'s of equation 2.2 are in different divided regions of $l_i$'s.
\end{assm}

\begin{lem}
Under assumptions 1 and 2, there are two properties of the solution of equation 2.8: (1) each subdomain $D_{i}$ of equation 2.2 should activate at least one hyperplane; (2) $D_{\nu}$ should be distinct from $D_{\mu}$ in at least one activated hyperplane for $1 \le \nu, \mu \le N$ with $\nu \ne \mu$.
\end{lem}
\begin{proof}
This lemma is an immediate consequence of the above two assumptions. The first property ensures an adjustable hyperplane to fit $F_i$ of equation 2.4 associated with $D_i$, and the second property makes the hyperplanes for different subdomains distinct.
\end{proof}

\begin{thm}
To interpolate a piecewise linear function of equation 2.1 or 2.4 by network $n^{(1)}m^{(1)}1^{(1)}$, under assumptions 1 and 2, the activation mode of any solution could be expressed in the form
\begin{equation}
\mathcal{M}^* = \begin{bmatrix}
\boldsymbol{P} & \boldsymbol{0}' & \boldsymbol{0}'  & \cdots & \boldsymbol{0}'\\
\boldsymbol{U} & \boldsymbol{P} & \boldsymbol{0}'  & \cdots & \boldsymbol{0}'\\
\vdots & \vdots & \vdots &\ddots & \vdots\\
\boldsymbol{U} & \boldsymbol{U} & \boldsymbol{U} & \cdots & \boldsymbol{P}
\end{bmatrix}.
\end{equation}
\end{thm}
\begin{proof}
Compared to lemma 4, in this case, each subdomain $D_i$ doesn't necessarily activate $n+1$ hyperplanes, and the $D_{\nu}$ and $D_{\mu}$ for $1 \le \nu, \mu \le N$ with $\nu \ne \mu$ need not to be discriminated by $n+1$ activated hyperplanes.

We construct the activation-mode matrix $\mathcal{M}$ according to lemma 5. To $D_1$ and $D_2$, we can arrange the order of the $m$ hyperplanes, such that $m_{11} = \boldsymbol{P}$, $m_{22} = \boldsymbol{P}$, and $m_{12} = \boldsymbol{0}$ or $m_{21} = \boldsymbol{0}$, which is the same as the case of lemma 4. But if $D_3$ is added, matrix $\mathcal{M}$ may be in this form
\begin{equation}
\mathcal{M}_3' = \begin{bmatrix}
\boldsymbol{P} & \boldsymbol{U} & \boldsymbol{U}\\
\boldsymbol{0} & \boldsymbol{P} & \boldsymbol{0'}\\
\boldsymbol{0}' & \boldsymbol{U} & \boldsymbol{P}
\end{bmatrix},
\end{equation}
where $m_{31} = \boldsymbol{0}'$ indicates that the number $n_3$ of the activated hyperplanes distinguishing between $D_3$ and $D_1$ needs not necessarily be the same as the number $n_2$ for $m_{21} = 0$ between $D_2$ and $D_1$. If $n_3 < n_2$, the definition of $\boldsymbol{0}'$ of equation 3.2 could express this relationship; otherwise, we can still use the original notation $\boldsymbol{0}$, which is a special case of $\boldsymbol{0}'$.

The remaining part can be similarly dealt with. For consistency of the notation, we write $m_{12} = \boldsymbol{0}'$ or $m_{21} = \boldsymbol{0}'$, which could include the $\boldsymbol{0}$ case, and similarly for other $\boldsymbol{0}$ entries. Finally, due to the same reason as equation 3.5, the standard form of the activation-mode matrix is equation 3.6.
\end{proof}

Next result shows how we could reach the solution mode of equation 3.6 in more specific forms, which are closely related to applications.
\begin{prp}
Three factors could lead to the solution with entry $\boldsymbol{0}'$ of equation 3.6 distinct from entry $\boldsymbol{0}$ of equation 3.3. The first is that there exists some $F_i$ of equation 2.4 whose interpolation is redundant. The second is due to the dependency between $F_i$'s for dependent piecewise linear interpolations. The third is about the overparameterization solution of lemma 2 of \citet*{Huang2022a}.
\end{prp}
\begin{proof}
The proof consists of three cases, corresponding to the above three factors, respectively.

\textbf{Case 1}. Let $N_3 = n^{(1)}m^{(1)}1^{(1)}$. If $F_i$ is on a lower-dimensional hyperplane of $n+1$-dimensional space $\boldsymbol{X}_{n+1}$ whose dimensionality is less than $n$,  its interpolation doesn't need $n+1$ activated hyperplanes to determine a unique $n$-dimensional hyperplane. This could result in the flexible parameter setting associated with $\boldsymbol{0}'$ of equation 3.6.

In fact, denote a $k$-dimensional hyperplane $l_k$ for $k < n$ that can interpolate $F_i$ by parametric equation
\begin{equation}
\boldsymbol{x} = \boldsymbol{x}_0 + \sum_{j = 1}^{k}t_{j}\boldsymbol{\lambda}_{j},
\end{equation}
which is embedded in $\boldsymbol{X}_{n+1}$ and all the vectors are of size $(n+1) \times 1$. Let ${\boldsymbol{w}'}^T\boldsymbol{x}' - y + b = 0$ be the equation of an $n$-dimensional hyperplane $l_i \subset \boldsymbol{X}_{n+1}$ passing through $l_k$, which corresponds to the output
\begin{equation}
y = {\boldsymbol{w}'}^T\boldsymbol{x}' + b
\end{equation}
of network $N_3$, where $\boldsymbol{x}'$ is a vector of the $n$-dimensional input space. Let $\boldsymbol{x} = [\boldsymbol{x}'^T, y]^T$ and $\boldsymbol{w} = [\boldsymbol{w}'^T, -1]^T$. Then the equation of $l_i$ could be expressed as
\begin{equation}
\boldsymbol{w}^T\boldsymbol{x} + b = 0,
\end{equation}
which is an equation of $n$-dimensional hyperplanes of $\boldsymbol{X}_{n+1}$.

Substituting equation 3.8 into equation 3.10 gives $\boldsymbol{w}^T\boldsymbol{x}_0 + b =0$ and $\boldsymbol{w}^T\boldsymbol{\lambda}_{j} = 0$ for $j = 1, 2, \cdots, k$, which can be integrated into
\begin{equation}
\boldsymbol{A}\boldsymbol{w} = \boldsymbol{b},
\end{equation}
where $\boldsymbol{A} = [\boldsymbol{\lambda}_1, \boldsymbol{\lambda}_2, \cdots, \boldsymbol{\lambda}_k, \boldsymbol{x}_0]^T$ with size $(k+1) \times (n+1)$, and $\boldsymbol{b} = [0, \cdots, 0, -b]^T$ is a vector of size $(k+1) \times 1$.

In the output layer of $N_3$, we should realize $l_i$ of equation 3.10 subject to equation 3.11 by producing $\boldsymbol{w}'$ and $b$ of equation 3.9, through the method of lemma 1 or 2 of \citet*{Huang2022a}. Analogous to equation 4.3 of \citet*{Huang2022a}, we have
\begin{equation}
\boldsymbol{\mathcal{W}}\boldsymbol{\alpha}_i = \boldsymbol{b}_i,
\end{equation}
where $\boldsymbol{b}_i = [{\boldsymbol{w}'}^T, b]^T$ of size $(n+1) \times 1$ includes the unknown parameters of $l_i$; the $\nu$ entries of $\boldsymbol{\alpha}_i = [a_1, a_2, \cdots, a_{\nu}]^T$ are the output weights of the $\nu$ activated hyperplanes of subdomain $D_i$ associated $F_i$; and $\boldsymbol{\mathcal{W}}$ of size $(n+1) \times \nu$ is the linear-output matrix of $D_i$ (definition 8 of \citet*{Huang2022a}). The final output is obtained by adjusting the $\nu$ variables of $\boldsymbol{\alpha}_i$.

By equation 3.12, each entry of $\boldsymbol{w}'$ as well as $b$ of equation 3.9 could be represented as the linear combination of the variables of $\boldsymbol{\alpha}_i$. Thus, substituting equation 3.12 into equation 3.11 yields $k+1$ linear equations with the entries of $\boldsymbol{\alpha}_i$ as the $\nu$ unknowns, which could be not homogeneous due to the constant entry $-1$ of $\boldsymbol{w}$ of equation 3.10 or 3.11. Then it is possible to find a nonzero-vector solution of $\boldsymbol{\alpha}_i$ to interpolate $F_i$ when $\nu \ge k+1$, which means that at least $k+1$ activated hyperplanes would be enough. Since $k+1 < n+1$, the conclusion follows.

\textbf{Case 2}. Even if $F_i$ determines a unique $n$-dimensional hyperplane $h_i$, the correlation between $F_i$ and some other $F_{j}$ for $j \ne i$ may result in the realization of $h_i$ by less than $n+1$ activated hyperplanes. The details are as follows.

Let $y = \boldsymbol{w}^T\boldsymbol{x} + b$ be the hyperplane of $h_i$ to be implemented, where $\boldsymbol{x}$ is a point of the $n$-dimensional input space of $N_3$. We consider $\boldsymbol{\mathcal{W}}\boldsymbol{\alpha}_i = \boldsymbol{b}_i$ of equation 3.12 as the method of realizing $h_i$, where the $\mu$ entries of $\boldsymbol{\alpha}_i = [a_1, a_2, \cdots, a_{\mu}]^T$ with $\mu < n+1$ are the unknowns and $\boldsymbol{b}_i = [\boldsymbol{w}^T, b]^T$. In this case, the number of equations is greater than that of the unknowns of $\boldsymbol{\alpha}_i$. However, there would exist a solution if $\text{rank}({\boldsymbol{\mathcal{W}}}) = \text{rank}([{\boldsymbol{\mathcal{W}}}, \boldsymbol{b}_i])$. Due the arbitrary selection of $\boldsymbol{b}_i$, we could find some hyperplane satisfying this condition.

\textbf{Case 3}. By lemma 2 of \citet*{Huang2022a}, the redundant activated hyperplanes of $D_i$ besides some $n+1$ necessary ones could not influence the implementation of an $n$-dimensional hyperplane $h_i$ of $F_i$. This also makes the parameter setting related to $\boldsymbol{0}'$ flexible, analogous to case 1.
\end{proof}

\section{Overparameterization Solutions}
There have been some experimental observations and analyses on overparameterization solutions, such as \citet*{Allen-Zhu2019}, \citet*{Du2019}, and \citet*{Li2018}. Our result of this section emphasizes the theoretical background that could lead to the overparameterization solution of three-layer networks.

We have explained the mechanism of overparameterization from several aspects in our preceding works \citep*{Huang2022a,Huang2022b}. In \citet*{Huang2022a}, to a three-layer network, lemma 2 indicated that the redundant activated hyperplanes could not influence the implementation of a linear function on a certain domain.

To the deep-layer case \citep*{Huang2022a}, the overparameterization due to the affine-transform redundancy is manifested in two ways. The first is that for $n$-dimensional input space, more than $n$ units can still be capable of realizing an arbitrary affine transform (proposition 6), despite the excess ones. The second is that any number of redundant layers could be added if they can transmit some data in the sense of affine transforms (the proof of lemma 6).

In \citet*{Huang2022b}, the redundant units for affine transforms were used to formulate the characteristic of the architecture of an encoder, that is, the number of units decreases monotonically as the depth of the layer grows. This is a type of overparameterization solution for encoders.

This paper focuses on the overparameterization mechanism of three-layer networks from the viewpoint of interpolation matrices, which can explain some kinds of overparameterization solutions.

\subsection{Model Description}
Note that when the underlying piecewise linear function $f$ of equation 2.1 or 2.4 is unknown, which is in fact the usual case in practice, there's not only one way to decompose the points of $F$ of equation 2.4 into subsets. Different decomposition ways correspond to different interpolation solutions, respectively. Thus, we reformulate the model of section 2.1 to omit the information of piecewise linear components.

Denote a set of data points of $n+1$-dimensional space by
\begin{equation}
F = \{(\boldsymbol{x}_k, y_k); k = 1, 2, \cdots, \mathcal{C}\},
\end{equation}
where $\boldsymbol{x}_k \in \mathbb{R}^n$ and $y_k \in \mathbb{R}$, which are to be interpolated. Equation 4.1 can be regarded as a discrete function
\begin{equation}
f: D \to \mathbb{R},
\end{equation}
where the domain $D = \{\boldsymbol{x}_{k}; k = 1, 2, \cdots, \mathcal{C} \} \subset \mathbb{R}^n$, and the range is $R = \{y_{k}; k = 1, 2, \cdots, \mathcal{C} \} \subset \mathbb{R}$.

Compared to the model of equation 2.1 or 2.4, equation 4.1 or 4.2 assumes that the information of piecewise linear components is unknown, and is to be discovered or formulated. Then equation 2.8 for interpolations is modified to
\begin{equation}
\boldsymbol{\Psi}\boldsymbol{\alpha} = \boldsymbol{y},
\end{equation}
where the $\mathcal{C} \times m$ interpolation matrix
\begin{equation}
\boldsymbol{\Psi} = \begin{bmatrix}
\Phi_1(\boldsymbol{x}_1) & \Phi_2(\boldsymbol{x}_1) & \cdots & \Phi_m(\boldsymbol{x}_1)\\
\Phi_1(\boldsymbol{x}_2) & \Phi_2(\boldsymbol{x}_2) & \cdots & \Phi_m(\boldsymbol{x}_2)\\
\vdots & \vdots &\ddots & \vdots\\
\Phi_1(\boldsymbol{x}_{\mathcal{C}}) & \Phi_2(\boldsymbol{x}_{\mathcal{C}}) & \cdots & \Phi_m(\boldsymbol{x}_{\mathcal{C}})
\end{bmatrix}
\end{equation}
omits the classification of the points through their corresponding linear components, $\boldsymbol{y} = [y_1, y_2, \cdots, y_{\mathcal{C}}]^T$, and $\boldsymbol{\alpha} = [\alpha_1, \alpha_2, \cdots, \alpha_m]^T$.

Equation 4.4 can also be expressed in the form of equation 2.13, which is
\begin{equation}
\boldsymbol{\Psi}=\begin{bmatrix}
\boldsymbol{L}_1, & \boldsymbol{L}_2, & \cdots, & \boldsymbol{L}_m
\end{bmatrix},
\end{equation}
where
\begin{equation}
\boldsymbol{L}_{\nu} = \begin{bmatrix}
\Phi_{\nu}(\boldsymbol{x}_{1}), \Phi_{\nu}(\boldsymbol{x}_{2}), \cdots, \Phi_{\nu}(\boldsymbol{x}_{\mathcal{C}})
\end{bmatrix}^T
\end{equation}
for $\nu = 1, 2, \cdots, m$.

The next theorem is a source of multiple solutions of the interpolation of equation 4.1 or 4.2.
\begin{thm}
To the interpolation of data set $F$ of equation 4.1 of $n+1$-dimensional space, there's more than one way to decompose it into subsets, in the form of $F = \bigcup_{i=1}^N F_i$ of equation 2.4, such that each of its subset $F_i$ is on an $n$-dimensional hyperplane.
\end{thm}
\begin{proof}
In fact, by not only one way, the associated domain $D$ for $F$ of equation 4.2 could be divided into this form $D = \bigcup_{i = 1}^kD_i$ through $n-1$-dimensional hyperplanes, where each $D_i$ contains less than or equal to $n+1$ points and is on a same divided region, such that $F_i$ with respect to $D_i$ can always be interpolated by an $n$-dimensional hyperplane.
\end{proof}

\begin{rmk}
The effect of this theorem is double-edged. On one hand, it increases the number of solutions. On the other hand, if $F$ is sampled from a certain piecewise linear function, the very solution that we want may be missing.
\end{rmk}

\subsection{Main Result}
To the interpolation by network $n^{(1)}m^{(1)}1^{(1)}$, the overparameterization of this section means that the number of the hyperplanes or units of the hidden layer is greater than that of the data points to be interpolated; that is, $m > \mathcal{C}$ in terms of interpolation matrix $\boldsymbol{\Psi}$ of equation 4.4.

The following theorem investigates the overparameterization solution of a three-layer network. For integrity, the case of $m = \mathcal{C}$ is also embedded in it.
\begin{thm}
To solve equation 4.3 for $m \ge \mathcal{C}$, there are $k = \binom{m}{\mathcal{C}}$ possibilities of constructing a square interpolation matrix $\boldsymbol{\Psi}'$. For fixed $\mathcal{C}$, $k$ is a polynomial of order $\mathcal{C}$ with respect to variable $m$ (i.e., the number of units). When $m > \mathcal{C}$, among the $k$ selections, each nonsingular case of $\boldsymbol{\Psi}'$ corresponds to an infinite number of solutions.
\end{thm}
\begin{proof}
When $m = \mathcal{C}$, we have $k = 1$ and there's at most one solution. When $m > \mathcal{C}$, to the linear equations of equation 4.3, the number $m$ of the unknowns $\alpha_{\nu}$'s for $\nu = 1, 2, \cdots, m$ is larger than the number $\mathcal{C}$ of equations.

Select a $\mathcal{C}$-combination from the $m$ entries $\alpha_{\nu}$'s as the $\mathcal{C}$ unknowns to form a square interpolation matrix, that is, select $\mathcal{C}$ columns of the interpolation matrix $\boldsymbol{\Psi}$ of equation 4.4 to form a square submatrix $\boldsymbol{\Psi}'$. Then equation 4.3 is changed into
\begin{equation}
\boldsymbol{\Psi}'\boldsymbol{\alpha}' = \boldsymbol{y}',
\end{equation}
where $\boldsymbol{\alpha}'$ only contains the selected $\mathcal{C}$ unknowns, and
\begin{equation}
\boldsymbol{y}' = \boldsymbol{y} - \sum_{r \in S}\alpha_r\boldsymbol{L}_{r},
\end{equation}
in which $S$ is the set of the remaining unknowns that are not chosen, and $\boldsymbol{L}_{r}$ is defined in equation 4.5.

The nonsingular $\boldsymbol{\Psi}'$ of equation 4.7 ensures the existence of a solution, and the free change of $\alpha_r$'s of equation 4.8 contributes to an infinite number of solutions. The number of different $\mathcal{C}$-combinations of $\alpha_{\nu}$'s is $\binom{m}{\mathcal{C}}$. This completes the proof.
\end{proof}

\begin{rmk-2}
The redundancy of the units of the hidden layer not only increases the probability of finding a solution, but also makes the number of solutions increase significantly.
\end{rmk-2}

\begin{rmk-2}
By this theorem, a standard algorithm for the solution construction is natural. Check the singularity of matrix $\boldsymbol{\Psi}'$ of equation 4.7 for each of the $\binom{m}{\mathcal{C}}$ combinations, until a nonsingular case is encountered.
\end{rmk-2}

\section{Training Mechanism}
This section will integrate the training process into the solution finding under our interpolation framework. Some of the results that follow are not limited to a three-layer network and are also applicable to the deep-layer case, such as lemmas 6, 7 and 8.

\begin{assm}
We assume that the quadratic lose function is chosen for the training, and the reason had been explained in section 7.2 of \citet*{Huang2022b}.
\end{assm}

\begin{lem}
To the training of a neural network, the smaller the value of the loss function, the more precisely the network approximates $F$ of equation 4.1, and the optimum solution is the interpolation of $F$.
\end{lem}
\begin{proof}
This conclusion is obvious by the property of the quadratic loss function.
\end{proof}

\begin{lem}
All the parameter updating or modifications during the training are for the decrease of the loss function or the interpolation of $F$ of equation 4.1.
\end{lem}
\begin{proof}
This is due to the quadratic loss function, the mechanism of the error backpropagation of the training, and lemma 6.
\end{proof}

\begin{lem}
Each parameter updating of the training is obtained with simultaneously all the remaining parameters fixed as constants.
\end{lem}
\begin{proof}
This lemma is a consequence of the principle of partial derivatives for multivariate functions.
\end{proof}

Denote by
\begin{equation}
H = \{l_j; j = 1, 2, \cdots, m\}
\end{equation}
the set of the hyperplanes derived from the units of the hidden layer of $n^{(1)}m^{(1)}1^{(1)}$.
\begin{thm}
Let $F_{i}$ be a subset of $F$ of equation 4.1, whose corresponding subdomain is $D_i \subset D$ of equation 4.2. If $F_i$ is not on an $n$-dimensional hyperplane, then during the training, by adjusting the parameters of $H$ of equation 5.1, $D_i$ tends to be in different divided regions of $H$, with the associated $F_{i}$ being changed accordingly.
\end{thm}
\begin{proof}
In this case, an $n$-dimensional hyperplane cannot interpolate $F_i$, while each divided region of $H$ can only produce one hyperplane after the parameters having been fixed. Thus, the value of the loss function could become smaller if this theorem holds, since more divided regions can provide more hyperplanes to fit $F_i$; otherwise, lemma 7 would be violated.
\end{proof}

\begin{rmk-3}
By theorems 3 and 5, the training is helpful to produce a solution by decomposing $F$ in various ways.
\end{rmk-3}

\begin{rmk-3}
The modification of the parameters of $H$ due to the training could change the interpolation matrix, through which the solution for interpolations is explored.
\end{rmk-3}

\begin{thm}
Under a fixed interpolation matrix $\boldsymbol{\Psi}$ of equation 4.4, which is derived from the arrangement of $H$ of equation 5.1, the output layer of $n^{(1)}m^{(1)}1^{(1)}$ could be trained to search for a solution of equation 4.3 for $\boldsymbol{\alpha}$ on the basis of a certain decomposition of $F$ of equation 4.1.
\end{thm}
\begin{proof}
This conclusion is a direct consequence of lemmas 7 and 8 for the parameters of the output layer. The interpolation via a three-layer network is in terms of a piecewise linear function whose each component is defined or implemented on a divided region of $H$. Let
each $D_i$ of $D = \bigcup_{i}D_i$ be the subdomain corresponding to $F_i$ of a partition $F = \bigcup_{i}F_i$, with $D_i$ being on a same divided region. When $H$ is fixed, each $D_i$ and $F_i$ are determined and the remaining work is to find an optimum $n$-dimensional hyperplane to fit $F_i$, through adjusting the vector $\boldsymbol{\alpha}$ of equation 4.3 by the training.
\end{proof}

\begin{rmk}
Theorems 5 and 6 are for the hidden layer and output layer of a three-layer network, respectively, which together comprise the whole training procedure.
\end{rmk}

\begin{dfn}
To the interpolation of equation 4.1 via neural networks, the method of using the fixed parameters of hidden layers as in theorem 4 is called the space-domain way. And the method of modifying the parameters of hidden layers by the training as in theorem 5 is called the time-domain way.
\end{dfn}

\begin{thm}
There are two ways to explore a solution of equation 4.3, including the space-domain way and the time-domain way.
\end{thm}
\begin{proof}
Even if the network is not trained, we may get a solution of equation 4.3 by theorem 4. If not, according to theorem 5, the training could change the interpolation matrix $\boldsymbol{\Psi}$ of equation 4.4 by modifying $H$ of equation 5.1, such that the condition of theorem 4 may be satisfied. The cooperation of the two ways provides more possibilities for finding a solution.
\end{proof}

\section{Classification via Deep Layers}
We have two goals in this section. The first is to give a special case of the sparse-matrix principle that will be proposed in a later section, which helps to understand it. The second is to simultaneously present a new mechanism of classification or data disentangling via deep layers different from \citet*{Huang2020,Huang2022a,Huang2022b}.

The interference among hyperplanes \citep*{Huang2022a} is a main difficulty of the interpolation by neural networks, for which an advantage of deep layers is that those disturbances could be eliminated. \citet*{Huang2020} resorted to an exclusively designed subnetwork called ``T-bias'' to achieve that goal, which is, however, a strong constraint on both the network architecture and the parameter setting.

In \citet*{Huang2022a}, compared to the method of T-biases, the interference-avoiding principle was used instead, which could be realized by the usual architecture of multilayer feedforward neural networks. Yet, this operation should be done in a higher-dimensional space with the help of the units of other subspaces.

The method to be introduced in this section is distinct from the above two ways, which needs neither a T-bias nor a higher-dimensional space. As long as the layer is deep enough without adding new units in each layer, the input data could be disentangled.

\subsection{Preliminaries}
For a two-category data set $D$ of $n$-dimensional space, we'll use the terms ``$*$-point'' and ``$o$-point'' to represent the element of the two categories, respectively, such as the points in Figure \ref{Fig.2} marked by ``*'' or ``o''.

Part of the proof of the next lemma is from that of lemma 10 of \citet*{Huang2022a}, with the purpose different.
\begin{lem}
Let $l_{1i}$ for $i = 1, \cdots, n$ be the hyperplane corresponding to unit $u_{1i}$ of the hidden layer of network $n^{(1)}n^{(1)}1^{(1)}$. Suppose that $l_{11}$ classifies data set $D$ of the input space into two parts $D_1$ and $D_2$, with $D_1 \subset l_{11}^+$ and $D_2 \subset l_{11}^0$. Then by adjusting the parameters of the remaining hyperplanes, in the sense of affine transforms, $D_2$ could be mapped to an arbitrary single point of $l_{11}$, and simultaneously $D_1$ is transmitted to the next layer.
\end{lem}
\begin{proof}
The proof begins with an example of Figure \ref{Fig.2}a, for which the network architecture is $N = 2^{(1)}2^{(1)}1^{(1)}$. Lines $l_{11}$ and $l_{12}$ are from the units of the hidden layer. The set $D$ of the points of Figure \ref{Fig.2}a is divided into two parts by $l_{11}$, denote by $D = D_1 \cup D_2$ with $D_1 \subset l_{11}^+$ and $D_2 \subset l_{11}^0$. We arbitrarily choose a point $O \in l_{11}$ with $O = l_{11} \cap l_{12}$, and rotate $l_{12}$ to make $D_1 \subset l_{11}^+l_{12}^+$ and $D_2 \subset l_{11}^0l_{12}^0$.

After the affine transform $\mathcal{T}$ associated with the hidden layer of $N$, the point $O$ will become $\mathcal{T}(O) = [0, 0]^T$. By the map of the hidden layer, due to the zero-output property of a ReLU, $D_2$ is also mapped to $[0, 0]^T$, which means that $D_2$ becomes the single point $O$ in $\mathcal{T}(\mathbb{R}^2)$, where $\mathcal{T}(\mathbb{R}^2)$ is an affine transform of the input space. Thus, we could say that, by the hidden layer, $D_2$ is changed to be $O$ in the sense of affine transforms.  And $D_1$ is transmitted because $D_1 \subset l_{11}^+l_{12}^+$, which is also in the sense of affine transforms as discussed in \citet*{Huang2020,Huang2022a}.

To the case of $n$-dimensional space, we should choose an arbitrary point on $n-1$-dimensional hyperplane $l_{11}$ as the mapped single point $O$, and construct other $n-1$ hyperplanes passing through this point, as well as having the same classification effect as $l_{11}$. One construction method is as follows.

Let $\boldsymbol{w}_1^T\boldsymbol{x} + b_1 = 0$ be the equation of $l_{11}$, where
\begin{equation}
\boldsymbol{w}_1 = \begin{bmatrix}
w_{11}, w_{12}, \cdots, w_{1n}
\end{bmatrix}^T.
\end{equation}
The constructed $n-1$ hyperplanes together with $l_{11}$ can be combined into the matrix form
\begin{equation}
\boldsymbol{W}\boldsymbol{x} + \boldsymbol{b} = \boldsymbol{0},
\end{equation}
where
\begin{equation}
\boldsymbol{W} = \begin{bmatrix}
w_{11} & w_{12} & w_{13} & \cdots & w_{1n} \\
w_{11} & w_{12} + \varepsilon_1  & w_{13} + \varepsilon_2 & \cdots & w_{1n} + \varepsilon_{n-1} \\
w_{11} & w_{12} + \varepsilon_1^2  & w_{13} + \varepsilon_2^2 & \cdots & w_{1n} + \varepsilon_{n-1}^2 \\
\vdots & \vdots & \vdots &\ddots &  \vdots\\
w_{11} & w_{12} + \varepsilon_1^{n-1}  & w_{13} + \varepsilon_2^{n-1} & \cdots & w_{1n} + \varepsilon_{n-1}^{n-1}
\end{bmatrix}
\end{equation}
is designed to be nonsingular with $0 < \varepsilon_i < 1$ and $\varepsilon_i \ne \varepsilon_j$ for $i \ne j$ (theorem 4 of \cite{Huang2020}), where $i, j = 1, 2, \cdots, n$.

Let $\boldsymbol{x}_0$ be an arbitrary point of $l_{11}$, which is designated as the point $O$. Independently of $\varepsilon_i$'s of equation 6.3, the parameter setting of $\boldsymbol{b} = -\boldsymbol{W}\boldsymbol{x}_0$ could make all hyperplanes $l_{1i}$'s pass through $\boldsymbol{x}_0$. If $\varepsilon_j$'s are sufficiently small, all $l_{1i}$'s for $i \ne 1$ could classify the data as $l_{11}$ (\cite{Huang2020}). This completes the proof.
\end{proof}

\begin{lem}[Lemma 8 of \citet*{Huang2022a}]
Denote by $\boldsymbol{x}_0$ a point of $n$-dimensional space, and let $\boldsymbol{w}^T\boldsymbol{x} + b = 0$ be the equation of a hyperplane $l$ derived from a ReLU. Through an affine transform, $\boldsymbol{x}_0$ and $l$ become $\boldsymbol{x}_0'$ and ${\boldsymbol{w}{'}}^T\boldsymbol{x}' + b' = 0$, respectively. Then we have $\boldsymbol{w}^T\boldsymbol{x}_0 + b = {\boldsymbol{w}{'}}^T\boldsymbol{x}'_0 + b'$, implying $\sigma(\boldsymbol{w}^T\boldsymbol{x}_0 + b) = \sigma({\boldsymbol{w}{'}}^T\boldsymbol{x}'_0 + b')$, which means that the output of a hyperplane (or unit) with respect to a point cannot be affected by affine transforms.
\end{lem}

\subsection{Main Results}

\begin{figure}[!t]
\captionsetup{justification=centering}
\centering
\subfloat[Linearly nonseparable classification.]{\includegraphics[width=1.9in, trim = {4.7cm 4.0cm 3.5cm 1.9cm}, clip]{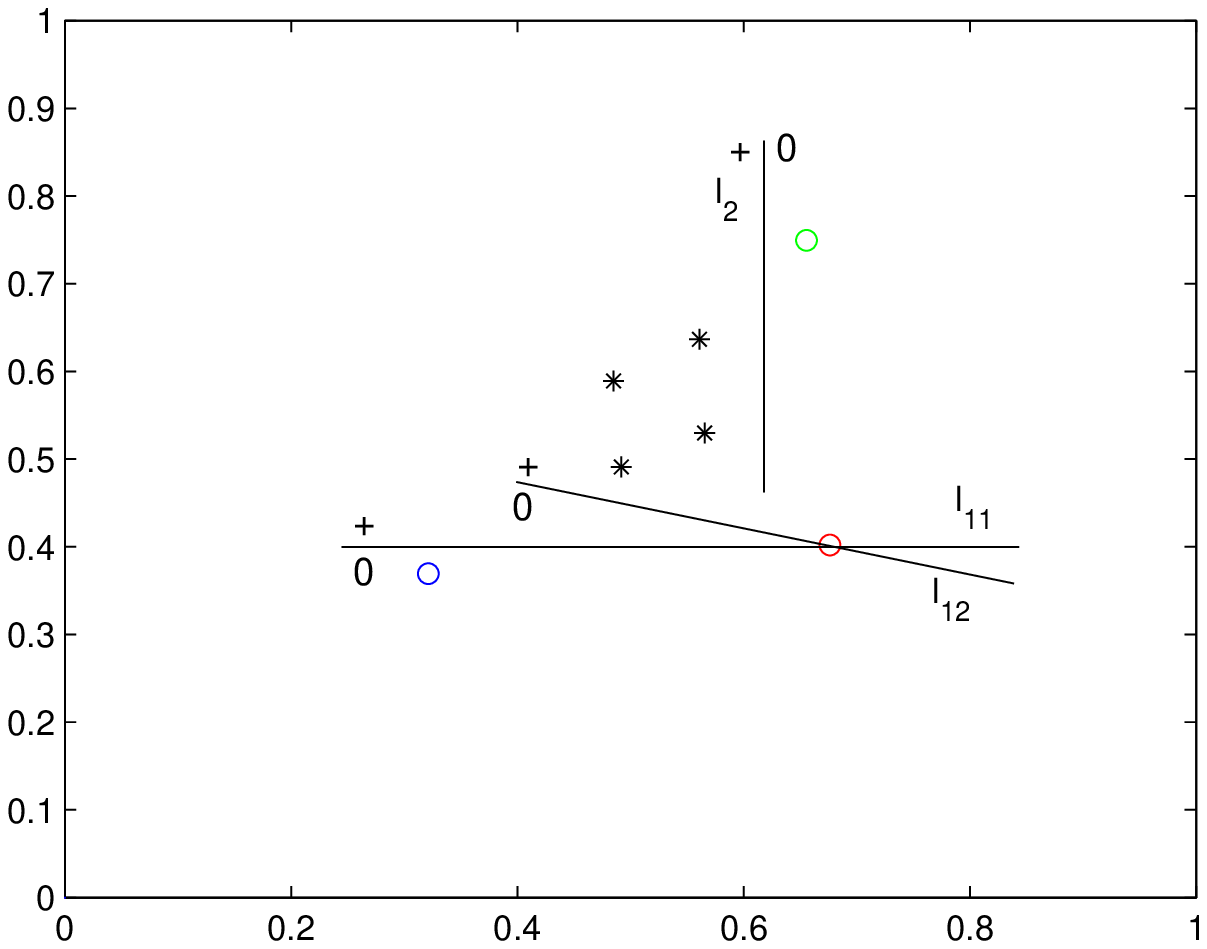}}
\subfloat[Another example for classification.]{\includegraphics[width=1.9in, trim = {4.8cm 4.0cm 3.8cm 2.1cm}, clip]{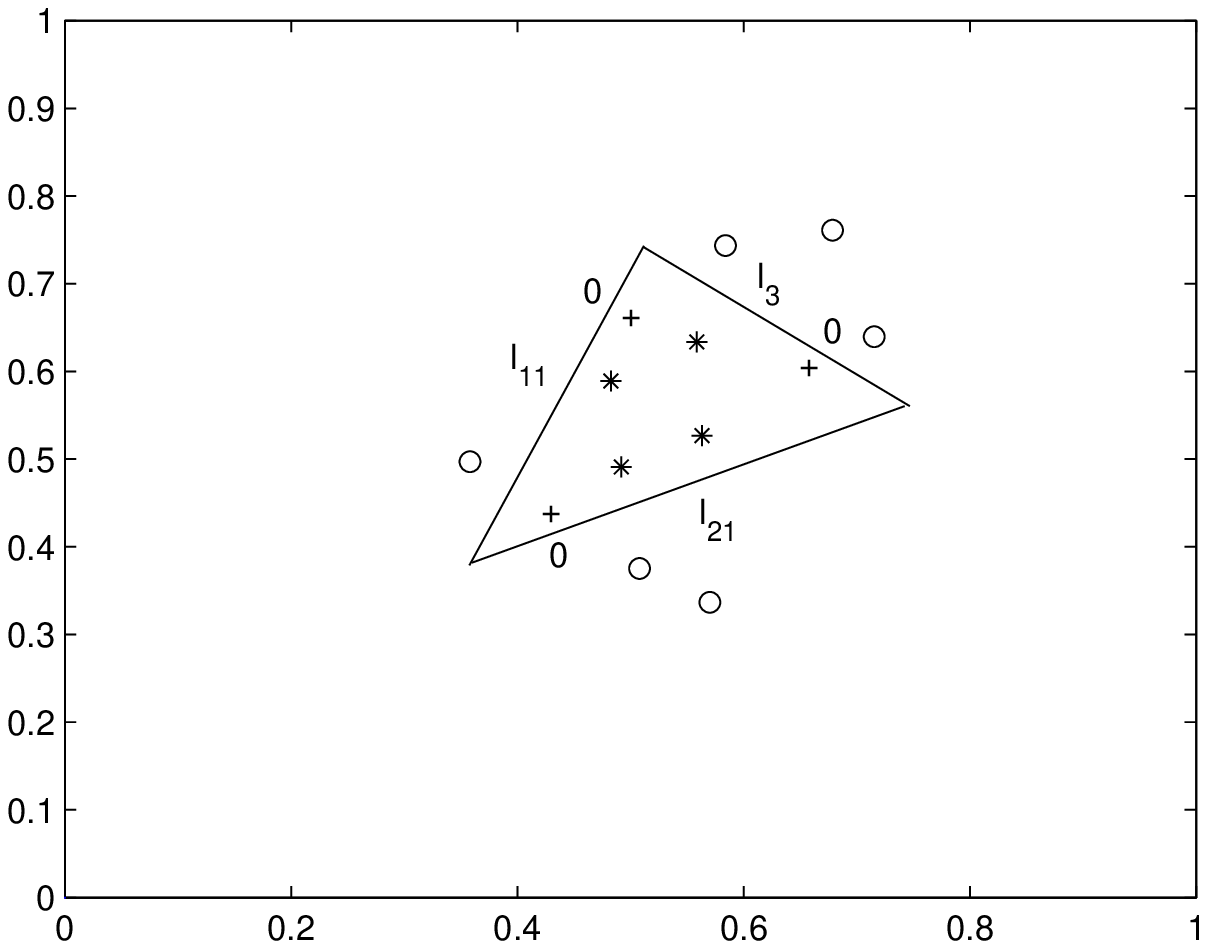}}
\subfloat[Network architecture for (b).]{\includegraphics[width=1.9in, trim = {4.3cm 4.9cm 3.8cm 2.5cm}, clip]{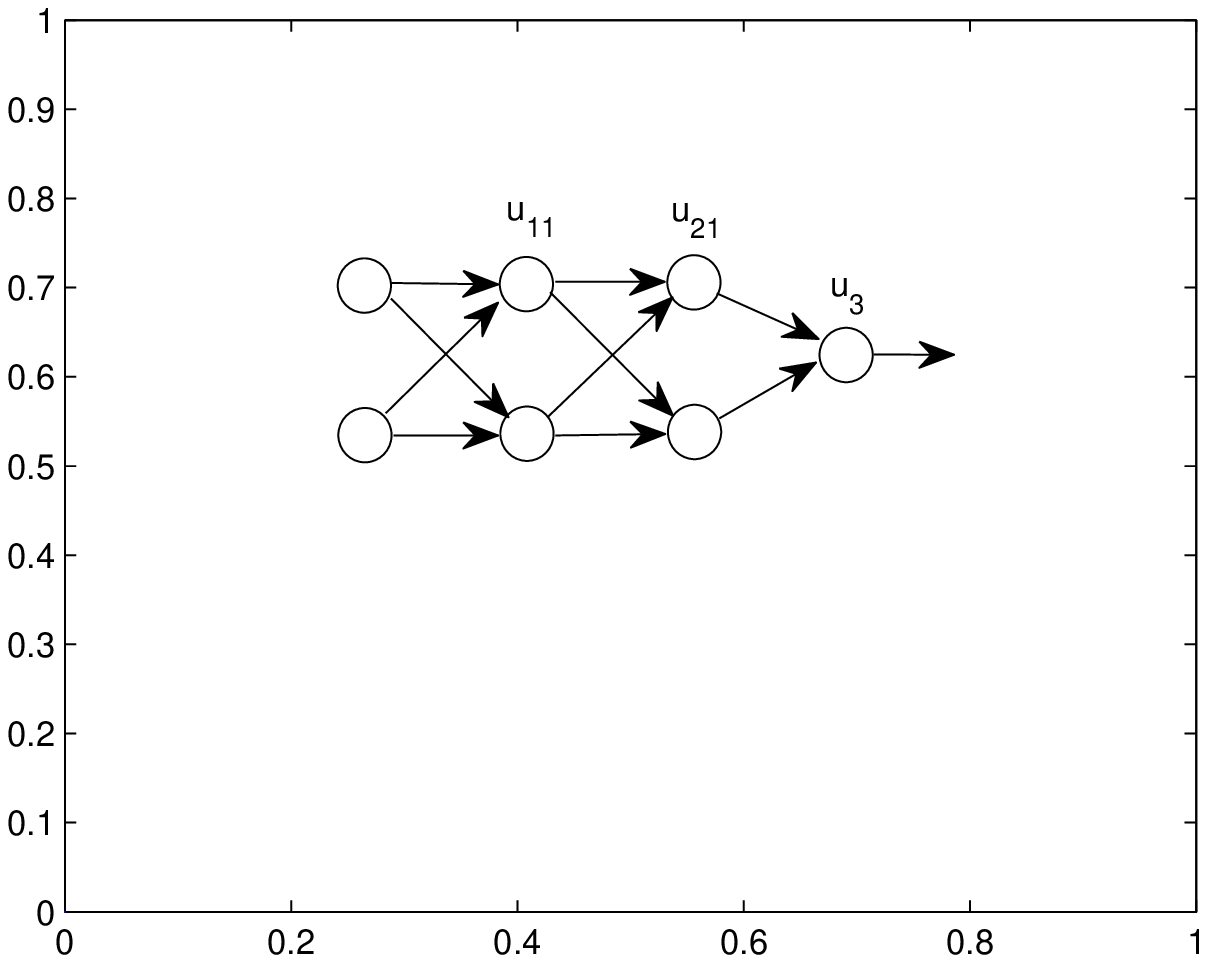}}
\caption{Effect of deep layers.}
\label{Fig.2}
\end{figure}

We use the terminology of a polytope of $n$-dimensional space from \citet*{Grunbaum2003}. Let \textsl{$n$-polytope} and \textsl{$k$-face} be an $n$-dimensional polytope and a $k$-dimensional face, respectively.

\begin{thm}
Given a two-category data set $D$ of $n$-dimensional space, if its $*$-points are contained in an open convex $n$-polytope $\mathcal{P}$ whose corresponding closed one has $d$ faces of dimensionality $n-1$, and if simultaneously its $o$-points do not belong to $\mathcal{P}$, then network $n^{(1)}n^{(d-1)}1^{(1)}$ can classify it.
\end{thm}
\begin{proof}
The proof is composed of three parts. Parts 1 and 2 are two examples that explain the main idea, and part 3 deals with the general case on the basis of the examples.

\textbf{Part 1}. In Figure \ref{Fig.2}a, the four black $*$-points and the three $o$-points with different colours are linearly nonseparable and to be classified by a three-layer network. Lines $l_{11}$ and $l_{12}$ correspond to the units of the hidden layer, whose intersection happens to be the red $o$-point; line $l_{2}$ is from the unit of the output layer. Because the blue $o$-point is in $l^{0}_{11}l^{0}_{12}$, it is mapped to the origin of the coordinate system of the hidden layer, where the red $o$-point is located in.

As can be seen, in the input space, regardless of the blue $o$-point, the red and green $o$-points are linearly separable from the $*$-points; due to the property of affine transforms, after passing through the hidden layer, they can still be linearly classified. Thus, by lemma 9, the points become linearly separable in the output layer, and a line such as $l_2$ can classify them.

In Figure \ref{Fig.2}a, line $l_2$ should have been depicted in an affine-transform region of $l^{+}_{11}l^{+}_{12}$; however, the effect is equivalent in the sense of linear classification due to the property of affine transforms, and similarly for other cases of the remaining proof.

\textbf{Part 2}. Another example is the repeated application of the above method through deep layers. The points of Figure \ref{Fig.2}b are to be classified by network $2^{(1)}2^{(2)}1^{(1)}$ of Figure \ref{Fig.2}c. Line $l_{ij}$ for $i,j = 1, 2$ of Figure \ref{Fig.2}b is derived from the $j$th unit of the $i$th layer of Figure \ref{Fig.2}c. For simplicity, lines $l_{12}$ and $l_{22}$ are omitted in Figure \ref{Fig.2}b. Line $l_3$ corresponds to the only unit of the third layer. The $*$-points are enclosed by a triangle formed by lines $l_{11}$, $l_{21}$ and $l_3$, while the $o$-points are excluded outside.

In Figure \ref{Fig.2}b, by lemma 9, through the first layer, we can map the $o$-point in $l_{11}^0$ (denote by $p_1$) to a point $p_1'$ of $l_{11}$ and simultaneously $p_1' \in l_{21}^0$. By lemma 10, after passing through the first layer, the relative position of $p_1'$ with respect to $l_{21}$, namely $p_1' \in l_{21}^0$, will be preserved by affine transforms.

Denote the two $o$-points in $l_{21}^0$ from left to right by $p_{21}$ and $p_{22}$, respectively. By the second layer, according to lemma 9, $p_1'$, $p_{21}$ and $p_{22}$ can be mapped to a point $p_2'$ such that $p_2' \in l_{21}$ as well as $p_2' \in l_3^0$. During the mapping process, the $*$-points are always made to be transmitted in the sense of affine transforms by the method of lemma 9. Finally in the third layer, all the $o$-points are mapped onto the region $l_3^0$, while the $*$ points are in $l_3^+$; that is, the points become linearly separable.

\textbf{Part 3}. The general case is similar. Let $\mathcal{P}'$ be the closed convex polytope corresponding to the open one $\mathcal{P}$ of this theorem, which has $d$ faces of dimensionality $n-1$. Then we need $d-1$ hidden layers and the network architecture is $n^{(1)}n^{(d-1)}1^{(1)}$. Each hidden layer $i$ for $i = 1, 2, \cdots, d-1$ corresponds to one $n-1$-face of $\mathcal{P}'$, which is denoted by $s_i$ and is part of the associated hyperplane $l_{i1}$ with respect to unit $u_{i1}$ of the $i$th layer.

For face $s_i$, we should select a point $O_i \in l_{i1}$ as the mapped point of the $o$-points in $l_{i1}^0$; and then construct other $n-1$ hyperplanes passing through $O_i$ and simultaneously having the same classification effect as $l_{i1}$, according to lemma 9.

The selection of $O_i$ for all $i$ should make the mapped version of $D$ linearly separable in the output layer. The method is similar to the example of part 2, which is as follows: From one $n-1$-face of the polytope $\mathcal{P}'$, say the $j$th one, the $o$-points in $l_{j1}^0$ are mapped onto the next adjacent region $l_{k1}^0$ by the selection of $O_j$, one after another, until the last one; finally, all the mapped $o$-points will be in a region $l_{d}^0$, where $l_d$ is the hyperplane associated with the single unit of the $d$th layer of $n^{(1)}n^{(d-1)}1^{(1)}$, corresponding to the last processed $n-1$-face of $\mathcal{P}'$.

Because $\mathcal{P}'$ is convex, in each map of the constructed hidden layers, the $*$-points can always be in a region transmitted by affine transforms, and will be in the destination of $l_d^+$. That is, finally, the points become linearly separable.
\end{proof}

\begin{cl}
Suppose that a two-category data set $D$ of $n$-dimensional space satisfies the condition of theorem 8. Then network $n^{(1)}\prod_{i=1}^dm_i^{(1)}1^{(1)}$ for $m_i > n$ could classify $D$ through the principle of both affine-transform generalizations \citep*{Huang2022a} and theorem 8.
\end{cl}
\begin{proof}
By theorem 8 of \citet*{Huang2022a}, network $n^{(1)}\prod_{i=1}^dm_i^{(1)}$ could be equivalent to $n^{(1)}n^{(d)}$ in realizing affine transforms. Combined with theorem 8 of this paper, the conclusion is proved.
\end{proof}

\section{Universal Solutions of Deep Layers}
This section investigates the mechanism of deep-layer networks in terms of interpolation matrices, on the basis of the three-layer case. Despite the close relationship of the two types of neural networks, we'll emphasize the distinction of deep layers in several aspects, including the solution, the activation, the property of interpolation matrices, and so on. The explanation of multi-output neural networks is in section 7.7.

\subsection{Preliminaries}
\begin{lem}
Given an $n-1$-dimensional hyperplane $l_n$ of $n$-dimensional space $\boldsymbol{X}_n$, let $l_k$ be another hyperplane whose dimensionality satisfies $1 \le k \le n-1$. Then the probability of $l_{k} \parallel l_n$ is $0$.
\end{lem}
\begin{proof}
By lemma 6 of \citet*{Huang2022b}, the probability of a line embedded in $\boldsymbol{X}_n$ parallel to $l_n$ is $0$. Let
\begin{equation}
\boldsymbol{x} = \boldsymbol{x}_0 + \sum_{j=1}^{k}t_j\boldsymbol{\lambda}_j
\end{equation}
be the parametric equation of $l_k$. Without loss of generality, we fixed the parameters of $t_j$'s except for $t_1$; then equation 7.1 could be written as $\boldsymbol{x} = \boldsymbol{x}_0' + t_1\boldsymbol{\lambda}_1$ with $\boldsymbol{x}_0'  = \boldsymbol{x}_0 + \sum_{\nu=2}^{k}t_{\nu}\boldsymbol{\lambda}_{\nu}$, which is a line denoted by $l_1$. Then $l_1 \subset l_k$. Line $l_1$ is not parallel to $l_n$ with probability $1$, and so is $l_k$, since $l_1 \subset l_k$.
\end{proof}

\begin{lem}
Under the notations of lemma 11, the probability of $l_k \subseteq l_n$ is $0$.
\end{lem}
\begin{proof}
The proof of lemma 6 of \citet*{Huang2022b} is only associated with the line direction, regardless of the line location. Thus, the probability of a line being on $l_n$ is also 0, since it is a special case of parallel-direction relationships. Then similar to the proof of lemma 11, the conclusion holds.
\end{proof}

\begin{prp}
Let $l_i$ for $i = 1, 2, \cdots, m$ be a hyperplane derived from the $i$th unit of the first layer of network $n^{(1)}m^{(1)}$. If $m \le n$, the probability of the existence of the divided region $\prod_{i=1}^ml_i^0$ is $1$.
\end{prp}
\begin{proof}
When $n=2$, the conclusion is trivial by lemma 11 and easy to be imagined. The case of $n = 3$ is general in the basic idea of the proof. When $m =2$, by lemma 11, planes $l_1$ and $l_2$ are probably unparallel to each other, and $l_1^0l_2^0$ obviously exists. If we add a third plane $l_3$, by lemmas 11 and 12, the line $l = l_1 \cap l_2 \subset l_1^0l_2^0$ intersects $l_3$ with probability $1$; that is, $l$ is not parallel to $l_3$ and $l \subsetneq l_3$. Due to the continuity of $l_1^0l_2^0$ on the neighbourhood of $l$, we have $l_3^0(l_1^0l_2^0) \ne \emptyset$.

The general case of $n$ can be similarly dealt with by inductive ways. Suppose that $\prod_{\nu = 1}^{k-1}l_{\nu}^0$ has been constructed, with $k < m$. By lemma 1 of \citet*{Huang2022b},
\begin{equation}
l = \bigcap_{\nu = 1}^{k-1}l_{\nu}
\end{equation}
is an $n-k + 1$-dimensional hyperplane. If a $k$th $n-1$-dimensional hyperplane $l_k$ is added, then $l$ is not parallel to $l_k$ as well as $l \subsetneq l_k$ with probability $1$, which implies the existence of $\prod_{i=1}^kl_i^0$. Repeat it until $k = m$.

Note that when $m = n$, the final $l$ of equation 7.2 is a point, which could not be further processed. Thus, this conclusion doesn't include the case of $m > n$.
\end{proof}

\begin{assm}
To network $n^{(1)}m^{(1)}$, we assume that the divided region $\prod_{i=1}^ml_i^0$ always exists, where $l_i$ is the $n-1$-dimensional hyperplane corresponding to the $i$th unit of the first layer.
\end{assm}

The rationale of assumption 4 lies in four aspects:
\begin{itemize}
\item[1.] Proposition 2 when $m \le n$.
\item[2.] When $m > n$, it is still possible that $\prod_{i=1}^ml_i^0$ exists.
\item[3.] Even if the condition of proposition 2 cannot be satisfied by a whole network, during its performance, some activated subnetwork may fulfil it, such as lemma 6 of \citet*{Huang2022a}.
\item[4.] In a deep-layer network, even if one layer satisfies assumption 4, it would be possible to influence the solution, whose mechanism will be discussed later.
\item[5.] We had given a construction method of $\prod_{i=1}^ml_i^0$ in theorem 4 of \citet*{Huang2020}.
\end{itemize}

\subsection{Model Description}
\begin{lem}
Denote a deep-layer network with $d$ hidden layers by
\begin{equation}
\mathcal{N} := n^{(1)}\prod_{i=1}^dm_i^{(1)}1^{(1)}.
\end{equation}
The interpolation of data set $F$ of equation 4.1 or 4.2 by $\mathcal{N}$ can be decomposed into two steps. The first is a map
\begin{equation}
f': D \to D',
\end{equation}
where $D'$ is the mapped data set of $D$ by the $d-1$th layer. The second step is an interpolation
\begin{equation}
\boldsymbol{\Psi}'\boldsymbol{\alpha} = \boldsymbol{y}
\end{equation}
by the three-layer subnetwork
\begin{equation}
\mathcal{N}_3 : = m_{d-1}^{(1)}m_{d}^{(1)}1^{(1)}
\end{equation}
of the last three layers of $\mathcal{N}$, where
\begin{equation}
\boldsymbol{\Psi}' = \begin{bmatrix}
\Phi_1(\boldsymbol{x}_1') & \Phi_2(\boldsymbol{x}_1') & \cdots & \Phi_{m_d}(\boldsymbol{x}_1')\\
\Phi_1(\boldsymbol{x}_2') & \Phi_2(\boldsymbol{x}_2') & \cdots & \Phi_{m_d}(\boldsymbol{x}_2')\\
\vdots & \vdots &\ddots & \vdots\\
\Phi_1(\boldsymbol{x}_{\mathcal{C}}') & \Phi_2(\boldsymbol{x}_{\mathcal{C}}') & \cdots & \Phi_{m_d}(\boldsymbol{x}_{\mathcal{C}}')
\end{bmatrix}
\end{equation}
with $\boldsymbol{x}_{k}' = f'(\boldsymbol{x}_{k})$ for $k = 1, 2, \cdots, \mathcal{C}$, $\boldsymbol{x}_{k} \in D$, $\boldsymbol{x}_{k}' \in D'$, and $\boldsymbol{\alpha} = [\alpha_1, \alpha_2, \cdots, \alpha_{m_d}]^T$.

Let
\begin{equation}
\mathcal{N} \setminus \mathcal{N}_3 := n^{(1)}\prod_{j=1}^{d-1}m_j^{(1)},
\end{equation}
which is obtained by deleting the last two layers of $\mathcal{N}$ of equation 7.3. Then $D'$ of equation 7.4 can be considered as the output of $\mathcal{N} \setminus \mathcal{N}_3$.

Under this model, correspondingly, equation 4.1 will become
\begin{equation}
F' = \{(\boldsymbol{x}_k', y_k); k = 1, 2, \cdots, \mathcal{C}\},
\end{equation}
which is to be interpolated by $\mathcal{N}_3$ of equation 7.6.
\end{lem}
\begin{proof}
The final output of $\mathcal{N}$ is directly produced by the subnetwork $\mathcal{N}_3$ of equation 7.6, whose mechanism is the case of three-layer networks. The input layer of $\mathcal{N}_3$ is the $d-1$th layer of of $\mathcal{N}$. The subnetwork $\mathcal{N} \setminus \mathcal{N}_3$ realizes the function $f': D \to D'$ of equation 7.4.
\end{proof}

The following assumption is reasonable, since we have discussed it in detail in \citet*{Huang2022b}.
\begin{assm}
Assume that the map $f'$ of equation 7.4 from the input layer to the $d-1$th layer of $\mathcal{N}$ is bijective.
\end{assm}

\begin{lem}
Under assumption 5, each map $f_i: D_{i-1} \to D_i$ for $i = 2, 3, \cdots, d-1$ from the $i-1$th layer to $i$th layer is bijective, where $D_0 = D$ and the input layer is regarded as the $0$th layer.
\end{lem}
\begin{proof}
Otherwise, if there exists some $f_i$ that is not bijective, the composite function $f' = f_{d-1} \circ f_{d-2} \circ \cdots f_0$ is also not bijective, where $f'$ is the map of equation 7.4, which is a contradiction of assumption 5.
\end{proof}

\subsection{Activation Route of Data}

\begin{figure}[!t]
\captionsetup{justification=centering}
\centering
\subfloat[Two non-overlapping routes.]{\includegraphics[width=2.4in, trim = {2.5cm 2.5cm 4.5cm 2.4cm}, clip]{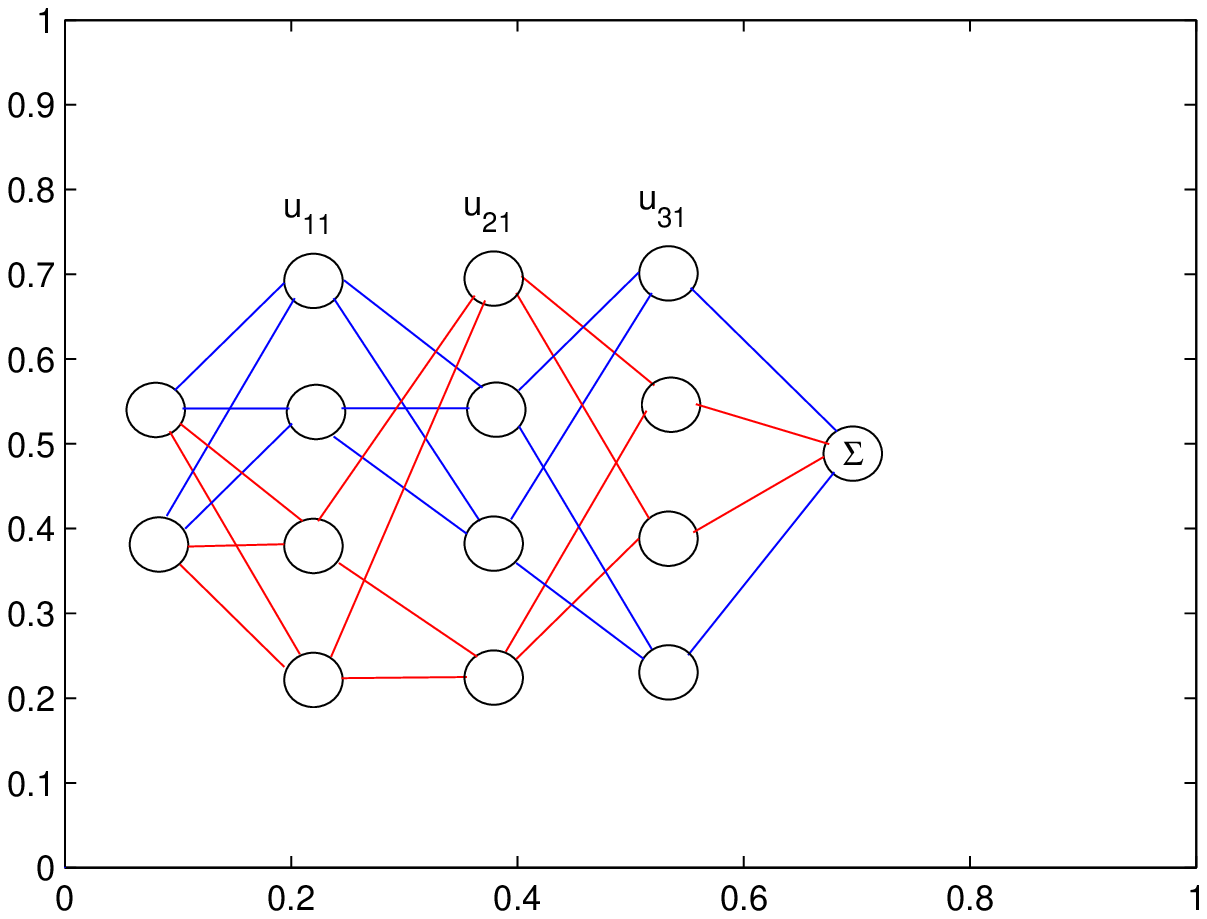}} \quad \quad
\subfloat[The third route.]{\includegraphics[width=2.4in, trim = {2.5cm 2.5cm 4.5cm 2.4cm}, clip]{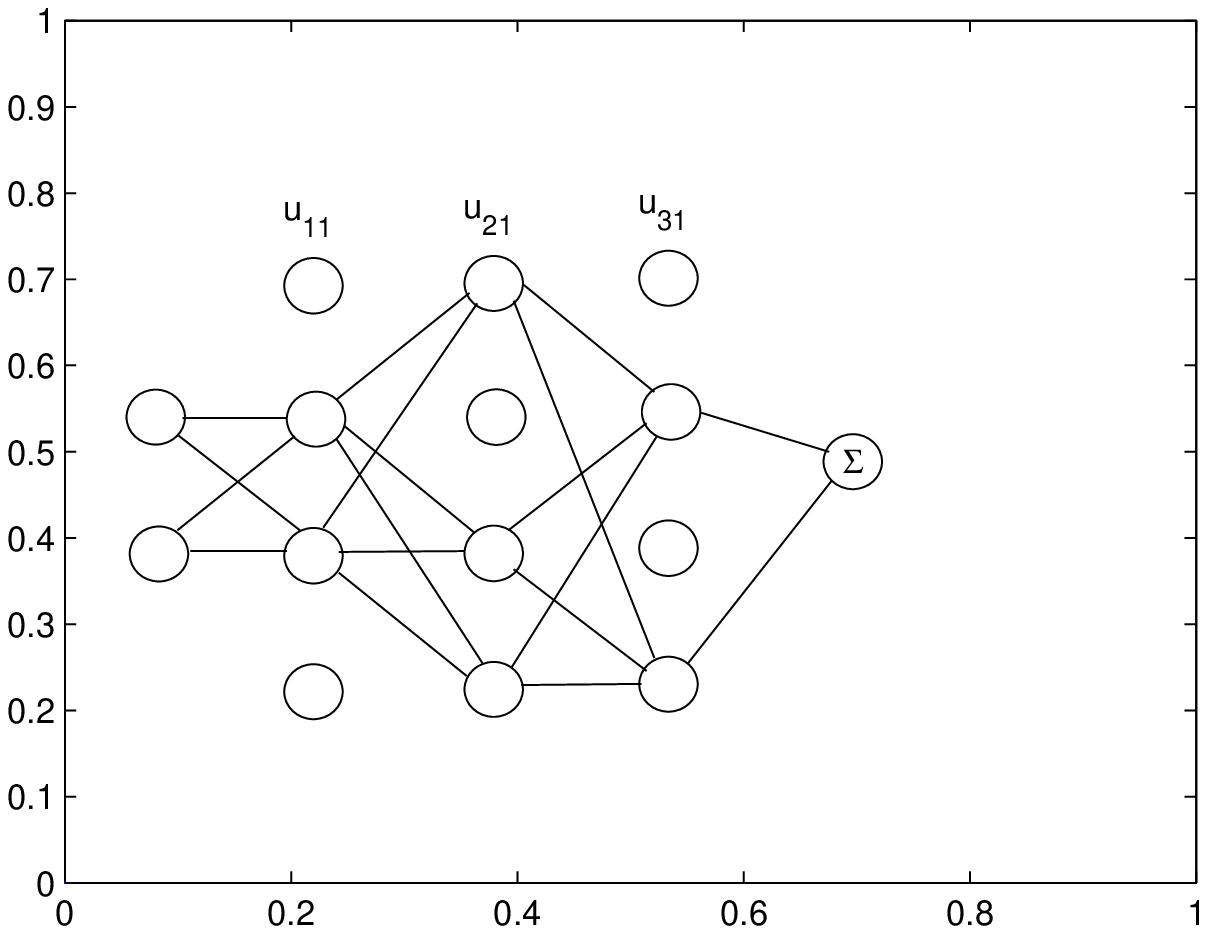}}
\caption{Activation routes of data.}
\label{Fig.3}
\end{figure}

\begin{dfn}
To a point $\boldsymbol{x} \in D$ of the input space, its route in network $\mathcal{N}$ of equation 7.3 is a subnetwork of $\mathcal{N}$ whose units are activated by $\boldsymbol{x}$ and are fully connected between adjacent layers.
\end{dfn}

For example, Figure \ref{Fig.3}a gives two different routes in $2^{(1)}4^{(3)}1^{(1)}$ connected by the blue and red links, respectively; and Figure \ref{Fig.3}b is another one with black links.

\begin{dfn}
The route of a subset $D_j \subset D$ in network $\mathcal{N}$ of equation 7.3 is the subnetwork that is composed of the units activated by any element of $D_j$, along with the fully connected links between adjacent layers. If a unit is in the route of $D_j$, there exists at least one point of $D_j$ activating it.
\end{dfn}

Denote the route of a point or subset of $D$ in network $\mathcal{N}$ of equation 7.3 by
\begin{equation}
\mathcal{R} := n^{(1)}\prod_{i=1}^{d}n_i^{(1)}1^{(1)},
\end{equation}
under which let
\begin{equation}
E_{\nu} := D^{(\nu)} \cap \prod_{j=1}^{n_i}l_{ij}^0
\end{equation}
with $\nu = i-1$, where $D^{(\nu)}$ is the mapped data set of $D$ by the $\nu$th layer (with $D^{(0)} = D$) of $\mathcal{N}$, and $l_{ij}$ is the hyperplane corresponding to the $j$th unit of the $i$th layer of route $\mathcal{R}$ of equation 7.10.

\begin{thm}
In route $\mathcal{R}$ of equation 7.10, the elements of $E_{\nu}$ of equation 7.11 will become a single point in subsequent layers. And the single point may vary layer by layer, which are denoted by $\boldsymbol{p}_{\nu+1}^{(\nu)}, \boldsymbol{p}_{\nu+2}^{(\nu)}, \cdots, \boldsymbol{p}_{d}^{(\nu)}$, with $\boldsymbol{p}_{d}^{(\nu)}$ the final output of the hidden layers of $\mathcal{R}$ for $E_{\nu}$. Especially, if $\boldsymbol{p}_{d-1}^{(\nu)} \in \prod_{j=1}^{n_d}l_{dj}^0$, then $\boldsymbol{p}_{d}^{(\nu)} = \boldsymbol{0}$, where $\boldsymbol{0}$ is a zero vector of size $n_d \times 1$.
\end{thm}
\begin{proof}
The proof is illustrated by an example. In Figure \ref{Fig.3}a, suppose that the subnetwork with blue links of $2^{(1)}4^{(3)}1^{(1)}$ is the route $r_1$ of $\boldsymbol{x}_1 \in D$. In the first layer of $r_1$, $\boldsymbol{x}_1$ activates units $u_{11}$ and $u_{12}$. By equation 7.11, $E_0 = D \cap l_{11}^0l_{12}^0$ doesn't activate $u_{11}$ and $u_{12}$. Thus, the outputs of the first layer of $r_1$ for all the elements of $E_0$ become a single point $\boldsymbol{p} = (0, 0)$; then in the subsequent layers, the outputs with respect to $E_0$ would always be a single point, until the last hidden layer.

The single point $\boldsymbol{p}$ for $E_0$ may be changed layer by layer, according to the parameter settings of each subsequent layer. And if in the second layer of $r_1$ we have $\boldsymbol{p} \in l_{31}^0l_{34}^0$, then the output of the last hidden layer of $r_1$ for $E_0$ is a zero vector.

The case of $E_1$ and $E_2$ is similar. The differences are that $E_1$ and $E_2$ become a single point from the second layer and third layer, respectively, and that the final output for $E_2$ is always a zero vector, since $E_2 \subset l_{31}^0l_{34}^0$. The general case is similar to this example.
\end{proof}

\begin{cl}
With the notations of theorem 9, if $\boldsymbol{p}_{k-1}^{(\nu)} \in \prod_{j=1}^{n_{k}}l_{kj}^0$ for $\nu+2 \le k \le d$, then the output of the last hidden layer of route $\mathcal{R}$ for $E_{\nu}$ is a zero vector.
\end{cl}
\begin{proof}
The condition of this corollary makes sure that the single output of each relevant hidden layer for $E_{\nu}$ is always a zero vector, until the last one.
\end{proof}

\begin{rmk}
The proof of theorem 8 of section 6 is an application of this corollary.
\end{rmk}

\begin{cl}
Without loss of generality, to $\boldsymbol{x}_1 \in D$, suppose that $n^{(1)}\prod_{i=1}^{d}n_i^{(1)}1^{(1)}$ of equation 7.10 is the route $r_1$ of $\boldsymbol{x}_1$ in network $\mathcal{N}$ of equation 7.3, and that set $E_{\nu}$ of equation 7.11 is for $r_1$. Let $\tau_{\nu} = |E_{\nu}|$ be the cardinality of $E_{\nu}$. Then in the interpolation matrix $\boldsymbol{\Psi}'$ of equation 7.7, the $\tau_{\nu}$ output vectors of the last hidden layer of $r_1$ for $E_{\nu}$ are the same.
\end{cl}
\begin{proof}
Select the $n_d$ columns of $\boldsymbol{\Psi}'$ of equation 7.7 associated with route $r_1$ as
\begin{equation}
\boldsymbol{\Psi}'_{r_1} = \begin{bmatrix}
\Phi_{k_1}(\boldsymbol{x}_1') & \Phi_{k_2}(\boldsymbol{x}_1') & \cdots & \Phi_{k_{\mu}}(\boldsymbol{x}_1')\\
\Phi_{k_1}(\boldsymbol{x}_2') & \Phi_{k_2}(\boldsymbol{x}_2') & \cdots & \Phi_{k_{\mu}}(\boldsymbol{x}_2')\\
\vdots & \vdots &\ddots & \vdots\\
\Phi_{k_1}(\boldsymbol{x}_{\mathcal{C}}') & \Phi_{k_2}(\boldsymbol{x}_{\mathcal{C}}') & \cdots & \Phi_{k_{\mu}}(\boldsymbol{x}_{\mathcal{C}}')
\end{bmatrix},
\end{equation}
where $\mu = n_d$ and $1 \le k_j \le m_d$ for $j = 1, 2, \cdots, n_d$, with $m_d$ the number of the units of the $d$th layer of $\mathcal{N}$.

We use an example of Figure \ref{Fig.3}a to show the main idea. Denote by $N_2 = 2^{(1)}4^{(3)}1^{(1)}$ the network of Figure \ref{Fig.3}a. Suppose that the subnetwork of $N_2$ connected by the blue links is the route $r_1$ of $\boldsymbol{x}_1$. The unit $u_{ij}$ for $i = 1, 2, 3$ and $j = 1, 2, 3,4$ is the $i$th unit of the $j$th layer of $N_2$, and $l_{ij}$ is the hyperplane corresponding to $u_{ij}$. For simplicity, only $u_{11}$, $u_{21}$ and $u_{31}$ are labeled in Figure \ref{Fig.3}. In this example, $n_d = 2$, and thus equation 7.12 becomes
\begin{equation}
\boldsymbol{\Psi}''_{r_1} = \begin{bmatrix}
\Phi_1(\boldsymbol{x}_1') & \Phi_4(\boldsymbol{x}_1')\\
\Phi_1(\boldsymbol{x}_2') & \Phi_4(\boldsymbol{x}_2')\\
\vdots & \vdots\\
\Phi_1(\boldsymbol{x}_{\mathcal{C}}') & \Phi_4(\boldsymbol{x}_{\mathcal{C}}')
\end{bmatrix},
\end{equation}
where $\Phi_1$ and $\Phi_4$ are associated with the outputs of $u_{31}$ and $u_{34}$, respectively.

To equation 7.11, Let $\tau_0 = |E_0|$. By theorem 9, the outputs of $u_{31}$ and $u_{34}$ for all the elements of $E_0$ are the same. Correspondingly, there would be $\tau_0$ identical rows in $\boldsymbol{\Psi}''_{r_1}$ of equation 7.13.

Let $D^{(1)}$ and $D^{(2)}$ be the mapped sets of $D$ by the the first layer and second layer of $N_2$, respectively. Write $E_1 = D^{(1)} \cap l_{22}^0l_{23}^0$ and $E_2 = D^{(2)} \cap l_{31}^0l_{34}^0$. Let $\tau_1 = |E_1|$ and $\tau_2 = |E_2|$. Analogous to the $E_0$ case, there will be $\tau_1$ identical rows in $\boldsymbol{\Psi}''_{r_1}$ for $E_1$, and $\tau_2$ identical rows for $E_2$. The general case is similar to this example.
\end{proof}

Let
\begin{equation}
H' = \{l_{dj}; j = 1, 2, \cdots, m_d\},
\end{equation}
where $l_{dj}$ is the hyperplane corresponding to unit $u_{dj}$ of the last hidden layer of network $\mathcal{N}$ of equation 7.3. According to the arrangement of $H'$, decompose $D'$ of equation 7.4 into the form
\begin{equation}
D' = \bigcup_{\mu}D_{\mu}',
\end{equation}
where each $D_{\mu}'$ is in a unique divided region of $H'$. Correspondingly, by assumption 5, because the map $f'$ of equation 7.4 is bijective, we have $D = \bigcup_{\mu}D_{\mu}$ of the input space, with
\begin{equation}
D_{\mu}' = f'(D_{\mu}).
\end{equation}

Since the linear components of a piecewise linear function realized by network $\mathcal{N}$ of equation 7.3 are discriminated by the divided regions of $H'$, the route of $D_{\mu}$ in $\mathcal{N}$ is particularly important, whose influence on the interpolation matrix $\boldsymbol{\Psi}'$ of equation 7.7 is characterized as follows.

\begin{cl}
Suppose that equation 7.10 is the route $r_{\mu}$ of $D_{\mu}$ of equation 7.16. Let
\begin{equation}
\boldsymbol{\Psi}'_{\mu} = \begin{bmatrix}
\Phi_1(\boldsymbol{x}_{1}') & \Phi_2(\boldsymbol{x}_{1}') & \cdots & \Phi_{n_d}(\boldsymbol{x}_{1}')\\
\vdots & \vdots &\ddots & \vdots\\
\Phi_1(\boldsymbol{x}_{\beta_1}') & \Phi_2(\boldsymbol{x}_{\beta_1}') & \cdots & \Phi_{n_d}(\boldsymbol{x}_{\beta_1}')\\
\vdots & \vdots &\ddots & \vdots\\
\Phi_1(\boldsymbol{x}_{\mathcal{C}}') & \Phi_2(\boldsymbol{x}_{\mathcal{C}}') & \cdots & \Phi_{n_d}(\boldsymbol{x}_{\mathcal{C}}')
\end{bmatrix}.
\end{equation}
be a submatrix of $\boldsymbol{\Psi}'$ of equation 7.7, where subscripts $1, 2, \cdots, n_d$ of each row represent the $n_d$ activated hyperplanes of $D_{\mu}'$ of equation 7.16, with $\beta_{\mu} = |D_{\mu}'| = |D_{\mu}|$. Then each $E_{\nu}$ of equation 7.11 with cardinality $\tau_{\nu} = |E_{\nu}|$ will contribute to $\tau_{\nu}$ identical rows in $\boldsymbol{\Psi}'_{\mu}$ of equation 7.17.
\end{cl}
\begin{proof}
The proof is trivial by corollary 3, with only the type of the route changed.
\end{proof}

\begin{rmk}
Note that corollaries 3 and 4 are two different view angles of the same interpolation matrix. Corollary 4 is more directly related to the piecewise linear components, which is the reason that we highlight it.
\end{rmk}

\subsection{Solution via Space-Domain Way}
The next two propositions are examples of the effect of deep layers on an interpolation matrix.
\begin{prp}
By \citet*{Huang2022a}, network $\mathcal{N}$ of equation 7.3 for $m_i > n$ could make the interpolation matrix $\boldsymbol{\Psi}'$ of equation 7.7 to be
\begin{equation}
\boldsymbol{\Psi}' = \begin{bmatrix}
\boldsymbol{B}_{11} & \boldsymbol{0} & \cdots & \boldsymbol{0}\\
\boldsymbol{0} & \boldsymbol{B}_{22} & \cdots & \boldsymbol{0}\\
\vdots & \vdots &\ddots & \vdots\\
\boldsymbol{0} & \boldsymbol{0} & \cdots & \boldsymbol{B}_{NN}
\end{bmatrix},
\end{equation}
which is a block-diagonal matrix, where the symbol $\boldsymbol{0}$ represents a zero matrix whose size varies according to the associated block.
\end{prp}
\begin{proof}
By the proofs of theorems 10, 5 and lemma 6 of \citet*{Huang2022a}, in the last hidden layer of $\mathcal{N}$, each unit can be constructed to be only activated by a certain subset $D_i$ of $D$; conversely, the subset $D_i$ of $D$ could have its distinct activation units and cannot activate others. Equation 7.18 indicates this fact by matrix $\boldsymbol{\Psi}'$.
\end{proof}

The proposition below derived from the results of section 6 demonstrates that the combination of different independent routes could yield complex behaviors.
\begin{prp}
Let $D = \bigcup_{j=1}^kD_j$ be a data set of the input space of network $\mathcal{N}$ of equation 7.3, where $D_j$ is contained in an open convex polytope that other points of $D$ do not belong to. Denote by $r_j$ the route of $D_j$ in $\mathcal{N}$, whose architecture is $n^{(1)}\prod_{i=1}^dm_i^{(1)}1^{(1)}$ for $m_i \ge n$. Then if $r_{\nu} \cap r_{\mu} = \emptyset$ for $1 \le \nu, \mu \le k$ with $\nu \ne \mu$, and if $r_j$ works by the mechanism of corollary 1 or theorem 8, the interpolation matrix $\boldsymbol{\Psi}'$ of equation 7.7 could be in the form of equation 7.18.
\end{prp}
\begin{proof}
In each route $r_j$, such as in Figure \ref{Fig.3}a, the method of corollary 1 or theorem 8 could separate $D_j$ from $D - D_j$, in terms of nonzero and zero output vectors of the last hidden layer of $r_j$, respectively; and the overall effect is the same as that of proposition 3.
\end{proof}

\begin{thm}[Sparse-matrix principle]
A deep-layer network can produce sparser interpolation matrix $\boldsymbol{\Psi}'$ of equation 7.7, compared to a three-layer one. And as the network becomes deeper, $\boldsymbol{\Psi}'$ has the possibility to be sparser.
\end{thm}
\begin{proof}
The proof consists of two parts, corresponding to the two conclusions of this theorem, respectively.

\textbf{Part 1}. Although proposition 3 has proved this capability of deep layers, we want to address this problem under more universal solutions.

To three-layer networks, the concept of half-space interference is proposed in definition 7 of \citet*{Huang2022a} from a geometric viewpoint; that is, each hyperplane or a unit of the hidden layer could influence half of the input space via nonzero outputs, which causes the difficulty of realizing an independent linear component of a piecewise linear function.

From the algebraic perspective of this paper, half-space interference could lead to many nonzero entries in each column of interpolation matrix $\boldsymbol{\Psi}$ of equation 4.4, rendering the property that whether $\boldsymbol{\Psi}$ singular or not more uncertain when $\mathcal{C} = m$, according to the principle of theorem 1.

This is the key point of an advantage of deep layers. Propositions 3 and 4 have provided the evidence or example that even when the points are influenced by some unwanted hyperplanes in the first layer, the disturbance can be eliminated by deep layers, in terms of zero entries of interpolation matrix $\boldsymbol{\Psi}'$ of equation 7.7.

There are two mechanisms that could make a deep-layer network to produce sparser matrix $\boldsymbol{\Psi}'$. First, in arbitrary layer $\nu$, there may exist the set $E_{\nu}$ of equation 7.11 with $\tau_{\nu} = |E_{\nu}|$, which could be mapped to a single point $\boldsymbol{p}$ in subsequent layers. By theorem 8 as well as corollaries 3 and 2, it is possible that all the elements of $E_{\nu}$ result in $\tau_{\nu}$ identical zero rows in the corresponding columns of $\boldsymbol{\Psi}'$. And we'll call this phenomenon \textsl{batch-zero production} through \textsl{point-overlapping operation}.

Each hidden layer has the possibility of doing \textsl{point-overlapping operation} for some points, and hence \textsl{batch-zero production} may occur more than one times, contributing to plenty of zero entries in $\boldsymbol{\Psi}'$.

Second, compared to $E_{\nu}$ of equation 7.11, its mapped single point $\boldsymbol{p}$ is much easier to have zero output in the last hidden layer of the associated route $\mathcal{R}$ of equation 7.10. As indicated in theorem 8 and proposition 3, the method of outputting a zero vector is based on the arrangement of hyperplanes to make set $E_{\nu}$ or point $\boldsymbol{p}$ in a certain divided region; and it's easier for a point to be in a divided region than for a set of points.

\textbf{Part 2}. The reason of the second conclusion is that the more the hidden layers, the more \textsl{point-overlapping operations} that the network may provide, resulting in the possible increase of zero entries of $\boldsymbol{\Psi}'$, due to \textsl{batch-zero production}.

As an example from theorem 8, when the depth of a network is large enough, all the points except for those we want to preserve could be excluded in terms of zero outputs, through \textsl{point-overlapping operation} and \textsl{batch-zero production}. This completes the proof.
\end{proof}

\begin{rmk}
The sparse $\boldsymbol{\Psi}'$ is closer to the ideal form of equation 7.18, and hence the solution could be more easily constructed or found.
\end{rmk}

\subsection{Time-domain Way}
To the solution of network $\mathcal{N}$ of equation 7.3, the time-domain way for its three-layer subnetwork $\mathcal{N}_3 = m_{d-1}^{(1)}m_{d}^{(1)}1^{(1)}$ of equation 7.6 has been discussed in theorems 5 and 6, and this section concentrates on subnetwork $\mathcal{N} \setminus \mathcal{N}_3$ of equation 7.8.

\begin{thm}
Under $H'$ of equation 7.14, a set of hyperplanes derived from the hidden layer of $\mathcal{N}_3$ of equation 7.6, suppose that $D_{\mu}' \subset D'$ of equation 7.15 is in a same divided region of $H'$, but the corresponding $F_{\mu}' \subset F'$ of equation 7.9 is not on a $m_{d-1}$-dimensional hyperplane. Then the training of $\mathcal{N} \setminus \mathcal{N}_3$ would enforce $D_{\mu}'$ to be separated into different divided regions of $H'$ by modifying the data set $D'$.
\end{thm}
\begin{proof}
By the condition of this theorem, if the parameters of $\mathcal{N}_3$ are fixed, $H'$ is determined and there's only one way to decrease the value of the loss function, that is, reformulating $D_{\mu}'$ to make its points scattered in different divided regions of $H'$. The training of $\mathcal{N} \setminus \mathcal{N}_3$ could change $D_{\mu}'$. By lemma 7, the conclusion follows.
\end{proof}

\begin{rmk-4}
The effect of this theorem can be manifested in the interpolation matrix $\boldsymbol{\Psi}'$ of equation 7.7.
\end{rmk-4}

\begin{rmk-4}
In comparison with a three-layer network alone, a deep-layer network could change the input data to its subnetwork of the last three layers, which adds a dimension of the optimization for interpolations.
\end{rmk-4}

\subsection{Sparse Activation Matrix}
The sparse-activation property of neural networks has been a focus of the study in both engineering and brain science, as described in literatures such as \citet*{Glorot2011}, \citet*{Goodfellow2016}, \citet*{Bengio2009}, and \citet*{Arpit2016}. Although the effect of $\ell_1$ penalty was emphasized in those preceding works, we'll show that it is not a necessary condition for the sparse activation.

\begin{dfn}
Let $D$ be a data set of the input space of network $m_1^{(1)}m_2^{(1)}$. The activation matrix of the first layer with respect to $D$ is defined as
\begin{equation}
\boldsymbol{\Lambda} = \begin{bmatrix}
\Phi_1(\boldsymbol{x}_1) & \Phi_2(\boldsymbol{x}_1) & \cdots & \Phi_{m_2}(\boldsymbol{x}_1)\\
\Phi_1(\boldsymbol{x}_2) & \Phi_2(\boldsymbol{x}_2) & \cdots & \Phi_{m_2}(\boldsymbol{x}_2)\\
\vdots & \vdots &\ddots & \vdots\\
\Phi_1(\boldsymbol{x}_k) & \Phi_2(\boldsymbol{x}_k) & \cdots & \Phi_{m_2}(\boldsymbol{x}_k)
\end{bmatrix},
\end{equation}
where $k = |D|$, $\boldsymbol{x}_i \in D$ for $i = 1, 2, \cdots, k$, and $\Phi_j$ for $j = 1, 2, \cdots, m_2$ is as defined in equation 2.7.
\end{dfn}

\begin{rmk}
Although the activation matrix is equal to the interpolation matrix of a three-layer network, the purposes of the two concepts are different. The former emphasizes the overall activation mode of a layer, while the latter is for interpolations.
\end{rmk}

\begin{dfn}
Let $L_j$ for $j = 0, 1, \cdots, d$ be the $j$th layer of deep-layer network $\mathcal{N}$ of equation 7.3. To each hidden layer $L_{\nu}$ for $1 \le \nu \le d$, its activation matrix is defined as equation 7.19, through only considering the two-layer subnetwork composed of the adjacent layers $L_{\nu-1}$ and $L_{\nu}$.
\end{dfn}

\begin{thm}
To deep-layer network $\mathcal{N}$ of equation 7.3, when data set $D$ of the input space passes through the hidden layers of $\mathcal{N}$, the activation matrix $\boldsymbol{\Lambda}_{\nu}$ of layer $L_{\nu}$ for $\nu = 1, 2, \cdots d$ has the possibility to be sparser as depth $\nu$ grows.
\end{thm}
\begin{proof}
Each hidden layer $L_{\nu}$ of $\mathcal{N}$ could be regarded as the last hidden layer of a network $N_{\nu}$, which is obtained by reducing the depth of $\mathcal{N}$ until $L_{\nu}$ becomes the last hidden one. Then the activation matrix $\boldsymbol{\Lambda}_{\nu}$ of $L_{\nu}$ is the interpolation matrix $\boldsymbol{\Psi}'$ of $N_{\nu}$, and hence the conclusion of $\boldsymbol{\Psi}'$ of theorem 10 could directly apply to $\boldsymbol{\Lambda}_{\nu}$.
\end{proof}

\begin{rmk}
Since $\ell_1$ penalty also results in the sparsity of matrix $\boldsymbol{\Lambda}_{\nu}$, we could see its effect on the solution of neural networks for interpolations.
\end{rmk}

\subsection{Multi-Output Case}
In theorem 3 and lemma 11 of \citet*{Huang2022a}, we investigated the parameter-sharing mechanism of neural networks for multi-outputs by solution constructions; and the conclusion is that the hidden layers are for the domain dividing, while each unit of the output layer realizes a piecewise linear function independently on shared divided subdomains. We here reconsider this problem in terms of interpolation matrices.

Denote a deep-layer network with multi-outputs by
\begin{equation}
\mathcal{N}' := n^{(1)}\prod_{i=1}^dm_i^{(1)}\mu^{(1)},
\end{equation}
which has $\mu$ units in the output layer. To each unit $u_{k}$ for $k = 1, 2, \cdots, \mu$ of the output layer of $\mathcal{N}'$, equation 7.5 becomes
\begin{equation}
\boldsymbol{\Psi}'\boldsymbol{\alpha}_k = \boldsymbol{y}_k,
\end{equation}
which shares the same interpolation matrix $\boldsymbol{\Psi}'$ for all $k$; and the difference lies in the input-weight vector $\boldsymbol{\alpha}_k$ and output vector $\boldsymbol{y}_k$. By equation 7.21, for fixed interpolation matrix $\boldsymbol{\Psi}'$, each $u_k$ works independently and has its own solution for $\boldsymbol{\alpha}_k$.

The three-layer case is a special type of equations 7.20 and 7.21, when $d = 1$ and $\boldsymbol{\Psi}' = \boldsymbol{\Psi}$ of equation 4.3. Due to $\boldsymbol{\Psi}'$ of equation 7.21, in general, the results of a single-output network of this paper are applicable to the multi-output one.

\section{Application in Autoencoders}
The mechanism of an autoencoder was intensively studied in \citet*{Huang2022b} and some solutions were provided. We want to further investigate it with the help of interpolation matrices, especially from the perspective of sparse activations.

\cite*{Bengio2009} and \citet*{Arpit2016} also discussed the sparse-representation problem of autoencoders and emphasized the regularization  way (such as $\ell_1$ penalty). Analogous to section 7.6, our investigation concentrates on the effect of deep layers rather than regularization. Accordingly, the effect of $\ell_1$ penalty on the solution of autoencoders could be interpreted by our results as the remark of theorem 12.

Denote the encoder $\mathcal{E}$ of an autoencoder by
\begin{equation}
\mathcal{E} := m^{(1)}\prod_{i=1}^dn_i^{(1)}n_e,
\end{equation}
where $m > n_{j-1} > n_j > n_e$ for $2 \le j \le d$, which realizes a map
\begin{equation}
f_e: D \to D_e,
\end{equation}
where
\begin{equation}
D = \bigcup_{i=1}^{N}D_i \subset \mathbb{R}^m
\end{equation}
is the input data set and $D_e$ is the output one; we regard $D$ as a multi-category data set, with its each subset $D_i$ of equation 8.3 being the $i$th category. Assume that the map $f_e$ is bijective, and that $|D| = \mathcal{C} = \sum_{i = 1}^N\beta_i$, where $\beta_i = |D_i|$.

Let
\begin{equation}
\mathcal{E}_3 := n_d^{(1)}n_e^{(1)}1^{(1)},
\end{equation}
where the first two layers are the last two ones of $\mathcal{E}$ of equation 8.1. The interpolation matrix of $\mathcal{E}_3$ is
\begin{equation}
\boldsymbol{\Psi}_e' = \begin{bmatrix}
\Phi_1(\boldsymbol{x}_{11}') & \Phi_2(\boldsymbol{x}_{11}') & \cdots & \Phi_{n_e}(\boldsymbol{x}_{11}')\\
\vdots & \vdots &\ddots & \vdots\\
\Phi_1(\boldsymbol{x}_{1\beta_1}') & \Phi_2(\boldsymbol{x}_{1\beta_1}') & \cdots & \Phi_{n_e}(\boldsymbol{x}_{1\beta_1}')\\
\vdots & \vdots &\ddots & \vdots\\
\Phi_1(\boldsymbol{x}_{N1}') & \Phi_2(\boldsymbol{x}_{N1}') & \cdots & \Phi_{n_e}(\boldsymbol{x}_{N1}')\\
\vdots & \vdots &\ddots & \vdots\\
\Phi_1(\boldsymbol{x}_{N\beta_N}') & \Phi_2(\boldsymbol{x}_{N\beta_N}') & \cdots & \Phi_{n_e}(\boldsymbol{x}_{N\beta_N}')\\
\end{bmatrix},
\end{equation}
where $\boldsymbol{x}_{\nu\mu}'$ for $\nu = 1, 2, \cdots, N$ and $\mu = 1, 2, \cdots, \beta_{\nu}$ is the element of
\begin{equation}
D' = \bigcup_{i=1}^{N}D_i',
\end{equation}
which is derived from the map $f': D \to D'$ via subnetwork $\mathcal{E} \setminus \mathcal{N}_2 := m^{(1)}\prod_{j=1}^{d}n_j^{(1)}$, analogous to equations 7.4 and 7.8.

\begin{lem}
If the interpolation matrix $\boldsymbol{\Psi}'_e$ of equation 8.5 could be normalized to the block-diagonal form of equation 7.18 by column-swap operations, such that each block $\boldsymbol{B}_{ii}$ corresponds to the nonzero outputs of $D_i'$ of equation 8.6, then the input data set $D$ is disentangled by encoder $\mathcal{E}$ of equation 8.1, and vice versa.
\end{lem}
\begin{proof}
The conclusion is obvious, since by the block matrix of equation 7.18, to each category $D_i'$, we could find a hyperplane to classify it from other ones, which means that $D$ is disentangled. Conversely, if $D_i'$ is linearly separable from other ones, then it can only activate some certain units exclusively for it, and simultaneously other categories cannot activate them, which is the basic characteristic of the block-diagonal matrix of equation 7.18.
\end{proof}

\begin{prp}
There exists a solution of encoder $\mathcal{E}$ of equation 8.1, such that to any input data set $D$ satisfying the embedding condition \citep*{Huang2022b}, the interpolation matrix $\boldsymbol{\Psi}'_e$ of equation 8.5 could be expressed in the block-diagonal form of equation 7.18, provided that the dimensionality $m$ of the input space is sufficiently large.
\end{prp}
\begin{proof}
The proof of theorem 8 of \citet*{Huang2022b} had constructed a solution of $\mathcal{E}$ to disentangle $D$. And by lemma 15, this proposition holds.
\end{proof}

\begin{thm}
Given an encoder $\mathcal{E}$ of equation 8.1 and a multi-category data set $D$ of its input space, it is possible that the deeper the encoder $\mathcal{E}$, the more likely it disentangles $D$.
\end{thm}
\begin{proof}
To some extent, the sparser the interpolation matrix $\boldsymbol{\Psi}'_e$ of equation 8.5, the closer it approaches the block-diagonal form of equation 7.18, under the constraint that each category has its own activated units. The principle of theorem 10 points out that a deeper neural network has the possibility to produce a sparser interpolation matrix; in combination with lemma 15, the conclusion follows. Proposition 5 is an example of this theorem.
\end{proof}

\section{Discussion}
This paper presented a general theoretical framework, and more detailed works, such as experimental verifications and further developments, should be done in future, with the ultimate goal of understanding the solution of engineering. The theory of this paper may not exactly match the reality, yet we expect that it could initiate the direction to the goal.

Under the framework, we have provided some basic principles for the solution construction in this series of researches, such as affine transforms \citep*{Huang2020,Huang2022a,Huang2022b}, the sparse-matrix principle (theorem 10), the mechanism of overparameterization solutions (\citet*{Huang2022a,Huang2022b} and theorem 4), the interference-avoiding principle \citep*{Huang2022a}. They not only are the examples of this general framework, but also may be the atomic components of the solution of engineering.

We also want to demonstrate the effectiveness of the deductive method in developing the theory of neural networks, analogous to the realm of theoretical physics. The best evidence is the explanation of experimental results in \citet*{Huang2022b}, and the fact that the universal solutions of some common network architectures of engineering could be obtained is also convincing. Future works combined with experimental analyses may provide more evidences.

When we gradually know more about a feedforward neural network, it is important to clarify that compared with other machine learning methods, what its distinct characteristic is and how this distinction could influence the performance. The contemporary popularity of deep learning doesn't mean that it's the destination or optimum of artificial intelligence.

Due to the relationship of a ReLU with a biological neuron \citep*{Glorot2011}, we hope that the series of papers \citep*{Huang2020,Huang2022a,Huang2022b} together with this one are also helpful to understand the brain science.

\end{document}